\documentclass[journal,twoside]{IEEEtran}
\ifCLASSINFOpdf
\else
\fi

\usepackage{amssymb}
\usepackage[colorlinks=False,linkcolor=blue]{hyperref}
\usepackage{amsmath}
\usepackage{graphicx}
\usepackage{caption}
\usepackage[dvipsnames]{xcolor}
\usepackage{subcaption}
\graphicspath{ {images/} }
\usepackage{float}
\usepackage{amsthm}
\usepackage{listings}
\usepackage{algorithm,algcompatible}
\makeatletter
\def\munderbar#1{\underline{\sbox\tw@{$#1$}\dp\tw@\z@\box\tw@}}
\makeatother
\newcommand{\norm}[1]{\left\lVert#1\right\rVert}
\usepackage{algorithm}
\usepackage{algorithmicx}
\usepackage{algpseudocode}
\usepackage{makecell}
\usepackage{booktabs}       
\usepackage{url}
\usepackage{multirow}
\usepackage{lipsum}
\usepackage{bbm}
\usepackage{soul}
\usepackage{mathtools}

\newtheorem{theorem}{Theorem}

\newtheorem{remark}{Remark}

\newtheorem{definition}{Definition}

\usepackage{lipsum}  
\usepackage{stmaryrd}

\usepackage{ upgreek }
\usepackage{textgreek}

\newcommand{\uu}{\mathbf{u}}
\newcommand{\s}{\mathbf{s}}

\newcommand{\R}{\mathbb{R}}
\newcommand{\Z}{\mathbb{Z}}

\usepackage{CJKutf8}
\usepackage{xcolor}
\usepackage{tikz}
\usetikzlibrary{shapes.geometric, arrows}
\usepackage{nomencl}
\makenomenclature

\begin{document}
%
\title{Rethink Repeatable Measures of Robot Performance with Statistical Query}

\author{Bowen Weng$^{1}$, Linda Capito$^2$, Guillermo A. Castillo$^{3}$, and Dylan Khor$^{4}$
\thanks{$^{1}$Department of Computer Science, Iowa State University, Ames, IA, USA;  {\tt\footnotesize bweng\@iastate.edu.}}
\thanks{$^{2}$Transportation Research Center Inc., East Liberty OH 43319, USA ;  {\tt\footnotesize capitol@trcpg.com.}}
\thanks{$^{3}$Independent researcher; {\tt\footnotesize gacastillom7@gmail.com.}}
\thanks{$^{4}$Department of Computer Science, Iowa State University, Ames, IA, USA; {\tt\footnotesize dkhor@iastate.edu.}}%
}

\maketitle

\begin{abstract}
For a general standardized testing algorithm designed to evaluate a specific aspect of a robot's performance, several key expectations are commonly imposed. Beyond accuracy (i.e., closeness to a typically unknown ground-truth reference) and efficiency (i.e., feasibility within acceptable testing costs and equipment constraints), one particularly important attribute is \emph{repeatability}. Repeatability refers to the ability to consistently obtain the same testing outcome when similar testing algorithms are executed on the same subject robot by different stakeholders, across different times or locations. 
However, achieving repeatable testing has become increasingly challenging as the components involved—testing algorithms, robotic platforms, measurement apparatuses, testing tasks, and environmental contexts—grow more complex, intelligent, diverse, and, most importantly, stochastic. While related efforts have addressed repeatability at ethical, hardware, and procedural levels, this study focuses specifically on achieving repeatable testing at the \emph{algorithmic} level.
Specifically, we target the most commonly adopted class of testing algorithms in standardized evaluation: statistical query (SQ) algorithms (i.e., algorithms that estimate the expected value of a bounded function over a distribution using sampled data). We propose a lightweight, parameterized, and adaptive modification applicable to any SQ routine—whether based on Monte Carlo sampling, importance sampling, or adaptive importance sampling—that makes it provably repeatable, with guaranteed bounds on both accuracy and efficiency.
We demonstrate the effectiveness of the proposed approach across three representative scenarios: (i) established and widely adopted standardized testing of manipulators, (ii) emerging intelligent testing algorithms for operational risk assessment in automated vehicles, and (iii) developing use cases involving command tracking performance evaluation of humanoid robots in locomotion tasks.

\end{abstract}


%
\IEEEpeerreviewmaketitle

\section{Introduction}\label{sec:intro}
The study of \emph{accuracy} and \emph{repeatability} has been a foundational aspect of robotics research since its inception~\cite{pieper1969kinematics,roth1976performance,sheridan1976performance,mooring1986determination,mooring1987aspects,todd2012fundamentals}. Early investigations into these concepts date back to the 1970s, initially within the context of industrial manipulators. Significant formalization of these measures emerged through international standards~\cite{dagalakis1999industrial} such as ISO 9283~\cite{iso9283}, first introduced in 1990, and ANSI/RIA R15.05~\cite{ansi_ria_r1505,jeswiet1995measuring}. Together, these standards provided comprehensive benchmarks that have globally shaped rigorous evaluation practices in industrial robotic manipulation. More recently, evaluation methodologies have evolved to encompass broader robotic domains beyond manipulation, extending into \emph{mobility} performance, as exemplified by standards such as ISO 18646~\cite{iso18646}, EURO New Car Assessment Program (NCAP) related testing methods for vehicles~\cite{van2017euro,euroncap2023aeb,rao2019tja}, and some undergoing efforts with testing methods and apparatus on mobile manipulators~\cite{aboul2022performance}.

At a high level, accuracy is defined as the closeness of agreement between a measured value and the true (reference) value. This definition is consistent with international standards such as ISO 5725-1:1994~\cite{ISO5725-1} and is widely adopted across metrology and robotics domains~\cite{NISTGlossary}. In contrast, repeatability refers to the degree of agreement between independent test results obtained under repeatability conditions. It is most commonly quantified either as a standard deviation of repeated results or as a repeatability limit (e.g., the interval within which 95\% of repeated measurements are expected to fall).

\begin{remark}\label{rmk:rp}
    Note that some early studies on accuracy and repeatability in robotics~\cite{pieper1969kinematics,roth1976performance,sheridan1976performance} characterize these attributes as inherent properties of the subject robots rather than of the testing algorithms themselves. While the underlying methodology can transfer between these perspectives, the primary focus of this paper is on the repeatability of the testing algorithms—that is, the ability to produce (almost) identical and accurate performance estimates for the same test subject, even when executed by different stakeholders under similar conditions.
\end{remark}

\begin{figure*}
    \centering
    \begin{subfigure}{.54\textwidth}
      \centering
      \includegraphics[trim={2cm 0cm 4cm 0cm}, clip, width=0.99\textwidth]{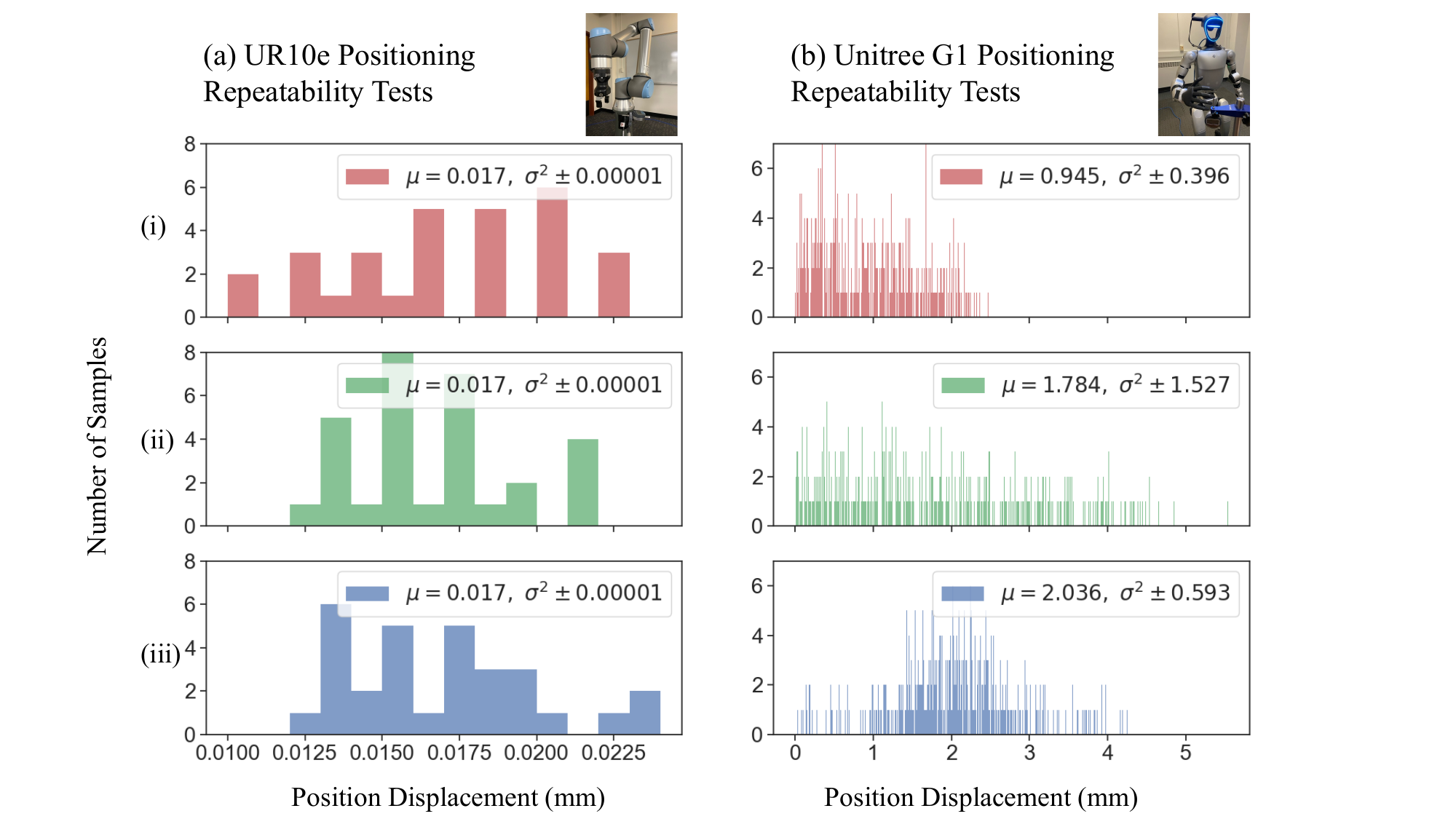}
      \caption{Applying the manipulation positioning repeatability testing procedure specified by ISO 9283 to two distinctly different robotic systems—a robotic manipulator (UR10e from Universal Robotics, shown in (a)) and a humanoid robot (G1 from Unitree Robotics, shown in (b))—each test was independently repeated three times (results illustrated in red, green, and blue in subfigures (i), (ii), and (iii), respectively). The testing procedure and testing equipment (a digital dial indicator with 0.001 mm resolution and a measuring range of 12.7 mm) remain consistent across the two robots and all tests with one notable exception: the UR10e manipulator undergoes 30 samples per trial (within the recommended ISO 9283 test range), while the Unitree G1 undergoes 250 samples per trial (five times the upper bound suggested by ISO 9283). The average displacement ($\mu$) and standard deviation ($\sigma^2$) for each trial and each robot are reported in their corresponding subfigures. The UR10e's algorithm and software are part of the commercial stack. The humanoid robot's control is developed in-house involving impedan ce whole-body control for in-place manipulation.}
      \label{fig:iso9283_np}
    \end{subfigure}%
    \hfill
    \begin{subfigure}{.43\textwidth}
      \centering
      \includegraphics[trim={5cm 0cm 7cm 0cm},clip,width=.99\textwidth]{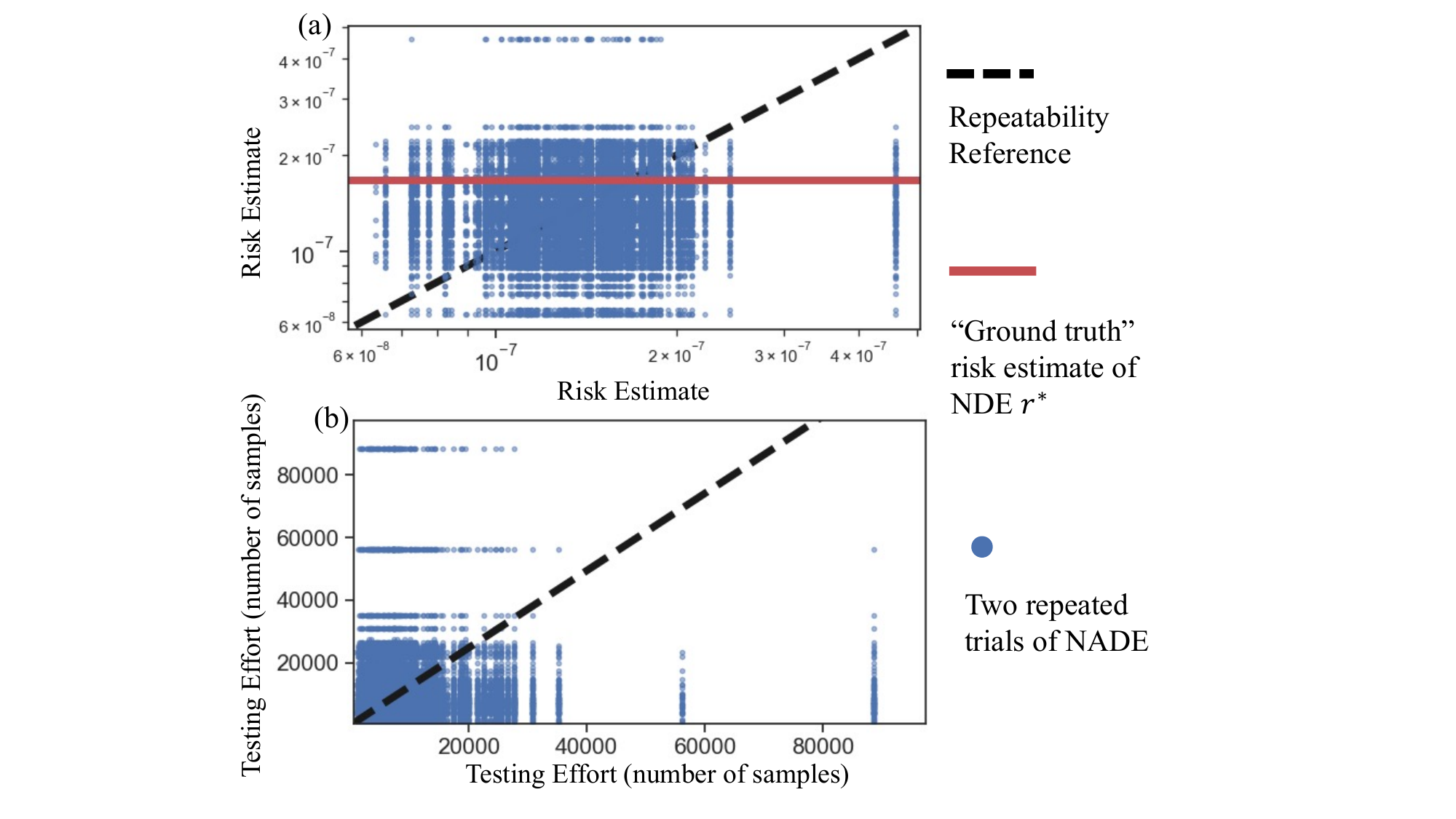}
      \caption{Each blue dot represents a pair of independent trials of executing NADE, an importance sampling inspired ADS risk assessment algorithm~\cite{feng2021intelligent}, against a pre-trained ADS. The execution of NADE strictly follows the open-source code~\cite{nadegithub}. The same results were also reported in Capito et al.~\cite{capito2024repeatable}, but not in the pairwise format. Subfigure (a) (top row) indicates the risk estimates obtained from two independent trials. Subfigure (b) (bottom row) indicates the number of samples each trial takes to converge to the desired relative half-width ($0.03$) with 0.05 confidence interval. The dark dashed line in each subfigure denotes the repeatability reference (i.e., two independent trials are having identical risk estimate if the blue dot overlaps the dashed line). For risk estimate, a ``ground truth" approximation obtained from extensive tests in NDE (Naturalistic Driving Environment) is also denoted in red.}
      \label{fig:nade_np}
    \end{subfigure}%
    \hfill
    \caption{Examples of \emph{non-repeatability} in the practice of robot performance testing. Notably, the issues highlighted by both of these examples are effectively addressed by the proposed method within their respective application domains, as will be demonstrated in Section~\ref{sec:exp}.}
    \label{fig:np}
    \vspace{-5mm}
\end{figure*}

Despite substantial empirical progress in the past decades, critiques have emerged regarding conventional definitions of repeatability in standards and testing. Traditional definitions implicitly assume that repeatability is equivalent to low variance under tightly controlled conditions. However, this assumption often does not hold, especially for physically more complicated and algorithmic-wise more intelligent robotic systems operating under dynamic or uncertain conditions. Consider the example illustrated in Fig.~\ref{fig:iso9283_np}, which adapts the standard positioning repeatability testing procedures from ISO 9283~\cite{iso9283} to two markedly distinct robotic systems: a robotic manipulator (Universal Robots UR10e equipped with a Robotiq e-gripper as its end-effector) and a humanoid robot (Unitree Robotics G1). Upon executing the complete testing procedure through three independent trials, the UR10e demonstrates remarkably consistent results, achieving an average positioning displacement of $0.017\text{ mm}$ with negligible variance (consistent up to the fifth decimal place) across all trials. In contrast, the humanoid robot G1 exhibits substantially varying average displacements and notable differences in variance among the three trials (as noted in Remark~\ref{rmk:rp}, the observed non-repeatability may stem in part from the robot’s own limitations (e.g., the G1's capabilities), but it also highlights the lack of repeatability in the testing procedure itself—adapted from ISO 9283—when applied to this particular application).

Moreover, the notion that repeatability is completely decoupled from accuracy is an oversimplification, and arguably misleading in the context of real robotic systems. While traditional definitions treat them as distinct (and orthogonal), in practice they are deeply interrelated. Shiakolas et al.~\cite{shiakolas2002accuracy}, for instance, constructed an error budget tree for industrial manipulators and introduced a metric called ``degree of influence" to quantify how each kinematic parameter variation contributes to both accuracy and repeatability. They found that manufacturing and link alignment errors impact the robot’s bias and variability simultaneously, and these factors must be analyzed together. Likewise, improving accuracy via calibration or control can sometimes trade off with repeatability: for example, a vision-guided correction system might reduce bias but introduce trial-to-trial noise. While classical metrics assume we can separately ``calibrate out" bias, in reality the line between bias and variance blurs once the robot is interacting with a non-static world.

From a theoretical standpoint, the statistical assumptions underlying repeatability metrics have drawn significant critique. Standard repeatability tests presume each trial to be an independent, identically distributed sample from a stationary error distribution. Brethé~\cite{brethe2010intrinsic} argues that this assumption is often invalid in practice, as real robotic systems commonly exhibit drift and non-stationary behaviors over repeated motions. Additionally, traditional repeatability metrics, such as those defined by ISO 9283, summarize errors with a single numeric value—typically related to the radius of a sphere encompassing 95\% of measured points. Researchers have noted that this simplification implicitly assumes an isotropic Gaussian distribution, which frequently does not reflect reality\cite{brethe2007risk}. Finally, when determining the appropriate testing effort, particularly the number of tests required to achieve reliable evaluation, traditional statistical significance-based methods exhibit inherent limitations. Such techniques typically rely heavily on assumptions about underlying data distributions and variance stability, rendering them sensitive and potentially misleading when those assumptions do not hold in practice~\cite{gat1995towards,montgomery2017introduction, brethe2010intrinsic}. Moreover, conventional ``online" approaches, notably those based on metrics like relative half-width (RHW)~\cite{feng2021intelligent,zhao2016accelerated}, have also been criticized for their reliability and robustness issues, particularly in adaptive or non-stationary testing environments~\cite{law2014simulation}, on the one hand, as they ``only guarantees probabilistic accuracy of the estimate, not repeatability of the testing outcome"~\cite{capito2024repeatable}. As shown in Fig.~\ref{fig:nade_np}—and similarly observed in Capito et al.\cite{capito2024repeatable}, albeit in a different format, an accelerated testing algorithm known as the Naturalistic Adversarial Driving Environment (NADE)~\cite{feng2021intelligent}, implemented using the original codebase provided by the authors~\cite{nadegithub}, was executed repeatedly in a pairwise fashion. The results revealed significantly divergent outcomes across trials, highlighting a notable lack of repeatability. On the other hand, this is especially relevant for contemporary robotic systems where testing conditions often evolve, data distributions change dynamically, and assumptions of independence and identical distribution (i.i.d.) are routinely violated. Consequently, such methods may underestimate the required testing effort, posing challenges to accurately quantify performance and reliability in real-world robotic scenarios~\cite{brethe2010intrinsic}.

As the broader robotics community transitions toward task-relative and process-oriented performance metrics~\cite{feng2023dense,kastner2022arena,aboul2022performance,sferrazza2024humanoidbench}, aiming to evaluate and understand robot behavior in conditions that genuinely reflect real-world usage rather than idealized spec-sheet scenarios, both the robotic systems under test and the testing tasks themselves are becoming increasingly complex, diverse, and intelligent. Within the research community, the predominant approach for replicable robotics has focused on adopting standardized hardware and software (the so-called ``one design'' approach~\cite{messina2019your}) and utilizing common benchmarks with relatively limited attention in addressing replicability at the algorithmic level, which still causes significant issues as particularly revealed in Fig.~\ref{fig:np}. Among standards bodies and regulatory entities for robotic products, facing extreme restrictions on feasible testing effort and cost~\cite{liu2022curse,euroncap2023aeb,forkenbrock2015nhtsa}, achieving statistical repeatability has always been challenging and is becoming increasingly difficult as the subject robots become increasingly intelligent, diverse, and stochastic.


\subsection{Literature Review}
While the above challenges of non-repeatability are primarily discussed in the context of standardized testing for robots, related issues are also explored in other fields, including machine learning~\cite{pineau2021improving,impagliazzo2022reproducibility,semmelrock2025reproducibility}, scientific discovery~\cite{mcnutt2014reproducibility,begley2015reproducibility,national2019reproducibility}, biology~\cite{errington2014open,voelkl2020reproducibility}, and psychology~\cite{bell2009repeatability}—often under different terminologies such as reproducibility, replicability, and pseudodeterminis (most notably in randomized algorithms~\cite{gat2011probabilistic,goldreich2019multi}). For example, the notion of repeatable testing algorithms closely parallels the broader scientific concern of ``Reproducibility and Replicability in Science," a topic that has received global attention and institutional support, notably from the National Science Foundation (NSF)~\cite{national2019reproducibility} at U.S.

In machine learning—particularly within computational theory~\cite{bun2020equivalence} and optimization—studies on reproducibility and replicability have largely focused on PAC (Probably Approximately Correct) learning frameworks~\cite{dixon2023list,vander2024replicability}, supervised learning methods, and optimization-based algorithms such as gradient descent~\cite{impagliazzo2022reproducibility}, convex optimization, and reinforcement learning convergence~\cite{esfandiarireplicable,karbasi2023replicability}. These works commonly adopt a definition of repeatability aligned with the one introduced later in Section~\ref{sec:prob}, and often employ geometric partitions~\cite{impagliazzo2022reproducibility,chase2023replicability,dixon2023list,chase2024local} to formalize the repeatability constraint. Our proposed use of $\alpha$-partitioning (see Section~\ref{sec:main}) builds upon this tradition but is tailored specifically for the testing applications addressed in this paper.

Moreover, existing analyses tend to prioritize complexity bounds, i.e., determining the number of samples or tests required, while giving less attention to efficiency, such as minimizing resource consumption or runtime costs. Research on reproducibility and replicability has also notably underexplored the Statistical Query (SQ) model, despite its central role in standardized, sampling-based evaluation pipelines. The few studies that do examine SQ reproducibility~\cite{impagliazzo2022reproducibility} typically focus on simplified Monte Carlo contexts and do not extend to more advanced or adaptive SQ mechanisms. This leaves a significant gap for further exploration into methods that simultaneously address accuracy, repeatability, and efficiency within the SQ framework.

Within the robotics field, studies on repeatability primarily focused on characterizing it as an inherent property of robotic manipulators—serving to complement other performance attributes such as accuracy~\cite{pieper1969kinematics,roth1976performance,sheridan1976performance,mooring1986determination,conrad2000robotic}. In contrast, the repeatability of specific testing algorithms (which is the central focus of this paper) has received comparatively limited attention. To the best of our knowledge, only a few recent efforts, including the work by Capito et al.~\cite{capito2024repeatable}, have explored this direction. However, their study is constrained to a specific class of SQ using importance sampling under known distributions. As a result, it does not address key challenges present in real-world robotic systems, such as performance drift, nor does it incorporate broader testing methodologies like adaptive sampling.

In other scientific disciplines such as biology~\cite{errington2014open,voelkl2020reproducibility}, psychology~\cite{bell2009repeatability}, and related fields, behavioral repeatability has also been extensively studied. At the algorithmic level, its definition closely parallels that of robotic repeatability discussed earlier in the context of manipulation tasks. Consequently, these domains encounter analogous issues and challenges, including variability across trials, sensitivity to initial conditions, and the influence of unmodeled external factors.

A final remark on related literature highlights that the study of repeatability, more broadly encompassing reproducibility and replicability, extends well beyond algorithmic considerations. It includes critical dimensions such as research ethics, standardized hardware platforms~\cite{messina2019your,yang2019replab}, software benchmarking frameworks~\cite{messina2019your,bonsignorio2015toward}, and community-driven practices for experimental documentation and data sharing~\cite{gunes2022reproducibility,falco2015grasping,faragasso2023reproducibility,sorbara2015testing,ullman2021challenges}. While these perspectives are not the primary focus of this paper, they are highly complementary to our proposed methodological framework. Importantly, as illustrated in Fig.\ref{fig:np}, even when established replicability practices are followed with consistent hardware and open-source software benchmark, repeatable testing outcomes are still not guaranteed due to inherent randomness and variability in system-environment interactions.

\subsection{Main contributions}
Motivated by these challenges, and inspired by a diverged set of methods from Lyapunov stability, concentration inequality, and geometric methods, this paper introduces a novel paradigm for achieving repeatably accurate performance measurement of robotic systems, particularly within the context of standardized testing procedures using SQ-type algorithms.

We propose \textbf{redefining repeatability} through the lens of stability, explicitly highlighting its theoretical trade-off with accuracy. A further key advantage of this approach is that it eliminates the need to repeatedly execute an identical testing algorithm (which is typically required by previous repeatable or replicable algorithms~\cite{capito2024repeatable}). Instead, as long as specified accuracy properties are consistently satisfied, repeatable testing outcomes can be reliably achieved, even when testing algorithms or procedures vary. This perspective naturally accommodates adaptive testing algorithms and accounts for performance drift or non-stationary behaviors in robotic systems, as discussed previously in the literature, thereby offering enhanced flexibility in evaluating increasingly diverse and intelligent robotic systems.

Furthermore, we develop \textbf{theoretical guarantees} that ensure effective and efficient testing procedures to reliably achieve the desired repeatable and accurate outcomes. Specifically, by incorporating the empirical Bernstein inequality~\cite{maurer2009empirical} into our analysis, we significantly reduce the required testing effort. 

Finally, our proposed algorithms and theoretical properties are \textbf{empirically demonstrated} across a variety of standardizable robot performance testing tasks, including risk assessment for automated vehicles, manipulation repeatability testing for robotic manipulators and humanoid robots, and trajectory or command-tracking capabilities of locomotive legged robots. This diverse range of applications highlights the compatibility and versatility of our proposed methods, demonstrating their effectiveness in complementing mature, widely-adopted testing standards, as well as in contexts where testing methodologies are currently underdeveloped or yet to be established. 

Our framework therefore addresses two complementary objectives: (i) advancing replicable robotics research through provably repeatable algorithms, and (ii) supporting standards organizations and regulatory bodies with practical tools for consistent performance evaluation.


\section{Preliminaries and Problem Formulation}\label{sec:prob}
For notations and nomenclature of this paper, one can refer to Appendix~\ref{apx:notation}.

\subsection{The testing system}
The general practice of testing a robot can be formulated as a dynamic control system
\begin{equation}\label{eq:ctrl-sys}
    \s(t+1) = f(\s(t), \uu(t); \omega(t)).
\end{equation}
The state $\s \in S \subset \R^s$ characterizes all observable properties of the testing participants and environmental conditions, such as the measured end-effector's positions in Fig.~\ref{fig:iso9283_np}, and the vehicles' positions and velocities involved in the tests revealed by Fig.~\ref{fig:nade_np}. The testing action $\uu \in U \subset \R^u$ denotes the controllable inputs of the testing system, such as the desired end-effector starting pos for each repeatably motion evaluation in Fig.~\ref{fig:iso9283_np}, and the driving behaviors of other traffic vehicles in the vehicle collision avoidance tests revealed by Fig.~\ref{fig:nade_np}. It is important to note that the ``action" does not directly correspond to the robot or the subject being tested; instead, it is encapsulated within the system dynamics $f$ (a detailed derivation can be found in \cite{weng2023comparability}). Additionally, $\omega \in W$ captures disturbances and uncertainties. This includes environmental randomness (e.g., wind speed or road friction) and algorithmic randomness, such as when the subject robot utilizes stochastic control algorithms.

Given any initial state $\s(0) \in S$ drawn from a probability distribution $p_s$ over the state space $S$, the \emph{testing} process generates a trajectory of states $T \subset S^{\xi}$ according to the testing system dynamics~\eqref{eq:ctrl-sys}. The trajectory begins at $\s(0)$ and propagates through a finite sequence of $\xi$ actions $\bar{\uu} = \{\uu(t)\}_{t \in \mathbb{Z}_{\xi}}$.

If these actions are predetermined and independent of the evolving state, they are termed \emph{open-loop} testing. Specifically, open-loop actions are sampled from a distribution $p_u$ over the action space $U^{\xi}$. Due to the independence between state and action samples, the joint distribution can be expressed as $p(x) = p(\s_0, \bar{\uu}) = p_s(\s_0) \cdot p_u(\bar{\uu})$ for some initial state $\s_0$ and action sequence $\bar{\uu}$.

Conversely, \emph{feedback} testing utilizes a testing policy $\pi$, where actions explicitly depend on the current state through $\uu(t) = \pi(\s(t))$. Integrating the feedback policy $\pi$—whether deterministic or stochastic—into the testing dynamics described by~\eqref{eq:ctrl-sys} implies that any uncertainties inherent in the policy itself are naturally encapsulated within the overall uncertainty term $\omega$, affecting the distribution of system trajectories accordingly.



\subsection{Performance estimate through statistical query}
A Statistical Query (SQ)~\cite{feldman2017general} is a query about an underlying data distribution that returns an estimate of the expectation of some bounded function with respect to that distribution. 

In the remainder of this paper, consider a testing algorithm for performance estimation $\mathcal{TE}$ as SQ, The underlying data distribution about which the algorithm is querying is $p$. The bounded function involved is a composition of the testing action propagation, trajectory collection, and single-trajectory evaluation. Given a collected set of state trajectories $T \subset S^{\xi}$, let $\psi: X \rightarrow \mathcal{M} \subset \R$ be the performance evaluation function, where $\mathcal{M} = [\underline{m}, \overline{m}]$ and $m = \overline{m} - \underline{m}$.

The ``ground-truth" of such a query is
\begin{equation}\label{eq:rstar}
    r^* = \mathbb{E}_{x \sim p}[\psi(x)]
\end{equation}
The most common interpretation of the queried outcome above is risk, particularly when $\psi$ is a Boolean-valued function indicating failure or non-failure. However, many other performance measures are also used in practice, depending on the testing objective and domain context.

Moreover, extended variants of SQ may not strictly sample from the original distribution $p$, but instead query under an alternative distribution $q$ defined over the same sample space. This is particularly useful in scenarios where the event of interest is sparse or rare~\cite{liu2022curse}, as a well-designed importance distribution $q$ can significantly accelerate the querying process and reduce the estimation variance~\cite{chatterjee2018sample}. This takes the form of 
\begin{equation}\label{eq:is}
    r_n = \frac{1}{n} \sum_{i=1}^n \psi(x_i)\frac{p(x_i)}{q(x_i)},
\end{equation}
with $x_1, \ldots, x_n \sim q(x)$.

It is immediate that when $p = q$,~\eqref{eq:is} reduces to the standard Monte Carlo estimator. As a further generalization, the samples may not be drawn from a fixed distribution $q$ throughout the querying process. Instead, one may update $q$ based on data-driven feedback—an approach known as adaptive importance sampling (AIS)~\cite{oh1992adaptive,bugallo2017adaptive,zhu2016gradient}. In some settings, the feedback testing policy (if applicable) can also be updated in an adaptive learning fashion to further improve efficiency or target rare events~\cite{feng2023dense,capito2020modeled}. Without loss of generality, this can be formulated as
\begin{equation}\label{eq:ais}
r_n = \frac{1}{n} \sum_{i=1}^n \psi(x_i)\frac{p(x_i)}{q_i(x_i)},
\end{equation}
where each sample $x_i \sim q_i(x)$. If $q_i = q$ for all $i$, then~\eqref{eq:ais} reduces to the standard importance sampling form in ~\eqref{eq:is}.

In the remainder of this paper, we also assume that the importance weight satisfies $w_i(x) = \frac{p(x)}{q_i(x)} \leq \overline{w}$ for all $x$ and all $i$ (with the subscript $i$ occasionally omitted for simplicity). This boundedness assumption is practically valid for all experiments considered in this study. We further note that, while boundedness is treated as an assumption in this paper, it can in fact be enforced in practice without compromising accuracy, for example, through techniques such as truncated importance sampling~\cite{ionides2008truncated}. A detailed treatment of such methods is beyond the scope of this paper.

Finally, for any testing algorithm that includes an SQ routine similar to~\eqref{eq:is} or~\eqref{eq:ais}, termination is governed by a criterion $\mathcal{T}: S^{\Z \times \xi} \rightarrow \mathbb{B}$, which maps the collection of all states from all tests to a Boolean decision. In this paper, the termination condition is designed to reflect statistical significance, ensuring that the results meet the desired levels of accuracy and reliability.

The above procedure are summarized in Algorithm~\ref{alg:is_testing}. For simplicity, we assume a fixed distribution $q$, as adapting $q$ introduces various alternatives—some of which will be explored later in Section~\ref{sec:exp}. Importantly, the theoretical results presented in this paper hold for both fixed and adaptive choices of $q$, which constitutes one of the key advantages of the proposed approach.

Intuitively, if the termination condition is removed and the algorithm is allowed to run indefinitely—or for a sufficiently large number of iterations—repeatability can be achieved naturally. This phenomenon is also illustrated in Fig.~\ref{fig:np}, where the averages across all trials appear concentrated near what is likely the true mean.

\begin{algorithm}[H]
    \begin{algorithmic}[1]
    \State {\bf Given: $X$, $p$, $q$, $\mathcal{T}$, $f$, $i=1, T=\emptyset$}
    \State {{\bf While} $\mathcal{T}(T)=0$:}
    \State {\ \ \ \ $(\s_0, \bar{\uu})_i = x_i \sim q(X)$}
    \State {\ \ \ \ Collect state trajectory $T_i$ through $f$~\eqref{eq:ctrl-sys}}
    \State {\ \ \ \ $r_i = \Big(\sum_{j=1}^i \psi(x_j) \cdot p(x_j)/q(x_j) \Big)/i$}
    \State {\ \ \ \ $T.\texttt{append}(T_i)$, $i$+=$1$}
    \State {{\bf Output}: $r_i$}
    \end{algorithmic}
    \caption{\footnotesize{$\mathcal{TE}_{IS}(X, p, q, \mathcal{T}, f)$} \label{alg:is_testing}}
\end{algorithm}

However, as noted in~\cite{capito2024repeatable}, ``this only guarantees the probabilistic accuracy of the testing algorithm rather than the repeatability of the testing algorithm." That is, while the estimate may be statistically close to the true value $r^*$, repeated executions of the algorithm do not necessarily produce the same output with high probability. Moreover, from a practical standpoint, achieving repeatability through excessive trials is infeasible, both in terms of cost and time. A key challenge, therefore, is to ensure that repeatability can be achieved efficiently, while balancing accuracy. To formalize this objective, we now introduce precise definitions of testing repeatability and accuracy.



\subsection{The repeatable and accurate testing problem}

The definition of accurate testing is straightforward and is formally stated as follows.
\begin{definition}\label{def:c-acc} [\textbf{$\gamma$-Accuracy}]
    A SQ-type performance estimate testing algorithm $\mathcal{TE}$ is $\gamma$-accurate if for some $c \in (0,1)$, $\gamma \in \R_{\geq0}$, and $r^*$ as defined in~\eqref{eq:rstar},  
    \begin{equation}
        \mathbb{P}(|\mathcal{TE}(\cdot) - r^*| \leq \gamma) \geq 1-c. 
    \end{equation}
\end{definition}

The definition of repeatable testing is given as follows.
\begin{definition}\label{def:beta-r} [\textbf{$\beta$-Repeatability}]
    A SQ-type performance estimate testing algorithm $\mathcal{TE}$ is $\beta$-repeatable if for some $\beta \in (0,1)$ and any two executions of the algorithm on the same testing system of~\eqref{eq:ctrl-sys}, $\mathcal{TE}^i$, $\mathcal{TE}^j$ and $i\neq j$, 
    \begin{equation}
        \mathbb{P}(\mathcal{TE}^i(\cdot) = \mathcal{TE}^j(\cdot)) \geq 1-\beta.
    \end{equation}
\end{definition}

The above definition is largely adapted from~\cite{capito2024repeatable}, following a similar framework to that of~\cite{impagliazzo2022reproducibility}, with one important distinction: the notion of the ``same algorithm" is significantly relaxed. Specifically, we do not require the two algorithms to be identical in every respect. For example, they may differ in implementation details, internal random seeds, sampling distributions, or computational environments, so long as they adhere to the same high-level testing specification. This relaxation is particularly well-suited to the practice of testing, where repeatability primarily concerns the output of the algorithm—as a characterization of the subject robot being tested—rather than the exact internal execution of the algorithm itself.

Under these definitions, it is straightforward to observe that a more ``accurate" algorithm also tends to be more ``repeatable". However, achieving high repeatability under stochastic conditions within limited testing effort is inherently challenging. This motivates the core idea of making a controlled trade-off between accuracy and repeatability. Consider an extreme example: let $g: X \rightarrow \mathcal{M}$ be a measurement function mapping into a compact set $\mathcal{M}$. Suppose we fix a point $m' \in \mathcal{M}$—for instance, the geometric center of $\mathcal{M}$—and define an algorithm that completely ignores the input $x$ and always returns $m'$. By construction, this algorithm is trivially repeatable, as its output is invariant. Yet it is clearly not accurate, since it fails to reflect any variability in $x$.

This example highlights the importance of carefully designing a discretization (or more generally geometric partitions) of the output space that balances accuracy with repeatability. Furthermore, the algorithm must also be efficient—consuming as few samples as possible—so that repeatable and informative measurements can be obtained within a reasonable testing budget. 

\section{Main Method}\label{sec:main}
The presentation of the proposed method begins with a general result showing how any $\gamma$-accurate testing algorithm can be transformed into a $\beta$-repeatable one. We then proceed to establish $\gamma$-accuracy of SQ-type testing algorithms under minimally restrictive assumptions. The section concludes with a variance-sensitive approach that jointly enables efficient, repeatably accurate testing.

\subsection{Testing repeatability}
Consider any testing algorithm $\mathcal{TE}(\cdot)$ defined above with a scalar output in $\mathcal{M} \subset \R$. Assume $\mathcal{TE}$ is $\gamma$-accurate with probability $1-c$ for some $c \in (0,1)$ and $\gamma \geq 0$ by Definition~\ref{def:c-acc}. 

We first define the almost uniform $\alpha$-partition as follows.
\begin{definition}\label{def:alpha-part}
    Given a closed interval $[a,b] \subset \R$, and a positive number $\alpha>0$,an almost-uniform $\alpha$-partition of the interval $[a,b]$ is defined as a finite collection of subintervals: $\{[x_0, x_1),[x_1, x_2),\dots,[x_{n-1},x_n]\}$, satisfying the following conditions:
    \begin{itemize}
        \item coverage: $a = x_0 < x_1 < \ldots < x_n = b$,
        \item uniformity: all subintervals except possibly the first $[x_0, x_1)$ and the last $[x_{n-1}, x_n]$ have length exactly equal to $\alpha$, and
        \item length constraint: each of the exceptional intervals (if they exist) must have length less than or equal to $\alpha$.
    \end{itemize}
\end{definition}
One common approach to constructing an almost uniform $\alpha$-partition, as suggested in~\cite{capito2024repeatable,impagliazzo2022reproducibility}, is to select the first interval uniformly at random from $[0, \alpha]$, followed by intervals of fixed length $\alpha$. This specific construction is also essential to the theoretical results in those works. In contrast, the methods proposed in this paper impose no such restriction on how the $\alpha$-partition is created; partitions may be defined via uniform sampling, manual specification, or any other suitable means.

The key idea for achieving repeatability is then to partition the output space into an almost uniform $\alpha$-partition. Rather than directly returning the original output, the algorithm returns a representative reference value of the interval into which the original output falls. We refer to this modified algorithm as the \emph{$\alpha$-quantized} version of the original algorithm. A formal definition using the midpoint as the reference return is given as follows.
\begin{definition}\label{def:alpha-quant}
Let $\mathcal{TE}(\cdot)$ be any SQ-type testing algorithm with scalar output in a bounded interval $\mathcal{M} = [\underline{m}, \overline{m}] \subset \R$. Let ${[x_0, x_1), [x_1, x_2), \dots, [x_{n-1}, x_n]}$ be an almost uniform $\alpha$-partition of $\mathcal{M}$, as defined in Definition~\ref{def:alpha-part}. The \emph{$\alpha$-quantized testing algorithm}, denoted $\mathcal{TE}_\alpha(\cdot)$, is defined by modifying the output of $\mathcal{TE}$ as follows:
\begin{equation}
    \mathcal{TE}_\alpha(\cdot) = \frac{x_{j-1} + x_j}{2}, \text{ such that } \mathcal{TE}(\cdot) \in [x_{j-1}, x_j).
\end{equation}
\end{definition}
This modification preserves the accuracy property of the original algorithm revealed by the following theorem.
\begin{theorem}\label{thm:quant-acc}
    Given a SQ-type performance estimate testing algorithm $\mathcal{TE}(\cdot)$ with a scalar output in $\mathcal{M} \subset \R$ that is $\gamma$-accurate (with probability $1-c$ for $c \in (0,1)$ and $\gamma \geq 0$). The $\alpha$-quantized version of $\mathcal{TE}(\cdot)$ (for some $\alpha>0$) is $\gamma+\frac{\alpha}{2}$-accurate (with probability $1-c$).
\end{theorem}
Note that the updated error bound $\gamma + \frac{\alpha}{2}$ arises from returning the midpoint of the corresponding subinterval. Among possible choices for the reference value, the midpoint introduces the smallest additional error due to $\alpha$-quantization, making it the most conservative and optimal in terms of minimizing quantization-induced distortion.

Consider two independent testing algorithm outputs $m_1$ and $m_2$ from a certain \emph{$\alpha$-quantized} execution of the testing algorithm. With a little abuse of notation, let them each independently with
\begin{equation}
    \mathbb{P}(|m_i-r^*|\leq \gamma) \geq 1-c, \forall i \in \{1,2\}.
\end{equation}

To achieve $\beta$-repeatability, let $D=|m_1-m_2|$. When both $m_1$ and $m_2$ lie within distance $\gamma$ from $r^*$, they are confined to the interval $[r^* - \gamma, r^* + \gamma]$ of length $2\gamma$. This geometric constraint implies that for any joint distribution of $(m_1, m_2)$ satisfying our conditions, the cumulative distribution of their distance $D$ must satisfy
\[
\mathbb{P}(D \leq d \mid |m_i - r^*| \leq \gamma) \geq \frac{4\gamma d - d^2}{4\gamma^2} \quad \text{for } d \in [0, 2\gamma].
\]
This bound arises because one cannot arrange points in the constrained region to have fewer close pairs than this formula allows. In $\alpha$-quantized execution with random offset, two points at distance $d$ fall in the same interval with probability
\[
\mathbb{P}_{\text{same}}(d) = \begin{cases}
\frac{\alpha - d}{\alpha} & \text{if } d \leq \alpha \\
0 & \text{if } d > \alpha
\end{cases}.
\]
Therefore, the probability of falling in the same interval is
\[
\mathbb{P}(\text{same interval} \mid |m_i - r^*| \leq \gamma) = \mathbb{E}\left[\frac{(\alpha - D)^+}{\alpha}\right].
\]
The minimum occurs when the distance distribution saturates the geometric bound. Thus
\[
\begin{aligned}
    \mathbb{P}&(\text{same interval} \mid |m_i - r^*| \leq \gamma) \\
    & \geq \int_0^{\alpha} \frac{\alpha - d}{\alpha} \, d\left(\frac{4\gamma d - d^2}{4\gamma^2}\right) \\
    & = \int_0^{\alpha} \frac{\alpha - d}{\alpha} \cdot \frac{4\gamma - 2d}{4\gamma^2} \, dd = \frac{4\gamma\alpha - \alpha^2}{4\gamma^2}
\end{aligned}.
\]
As a result,
\begin{equation}\label{eq:beta-1}
    \mathbb{P}(\text{same interval} \boldsymbol{\mid} |m_i-r^*|\leq \gamma) 
    \geq \frac{4\gamma\alpha-\alpha^2}{4\gamma^2}.
\end{equation}
Including the probability that both points satisfy the accuracy constraint, we thus have
\begin{equation}\label{eq:beta-2}
    (1-c)^2\frac{4\gamma\alpha-\alpha^2}{4\gamma^2} \geq 1-\beta.
\end{equation}
This further leads to the following theorem on repeatable testing.
\begin{theorem}\label{thm:quant-rep}
    Given a SQ-type testing algorithm $\mathcal{TE}(\cdot)$ with a scalar output in $\mathcal{M} \subset \R$ that is $\gamma$-accurate (with $\gamma > 0$ and probability $1-c$ for $c \in (0,1)$). The $\alpha$-quantized version of $\mathcal{TE}(\cdot)$ is $\beta$-repeatable ($\beta \in (0,1)$ and $(1-c)^2 \geq (1-\beta)$) for 
    \begin{equation}\label{eq:alpha}
       \alpha = 2\gamma\frac{(1-c)-\sqrt{(1-c)^2 - (1-\beta)} }{(1-c)}
    \end{equation}
\end{theorem}
The proof follows directly from~\eqref{eq:beta-2} and is elaborated in detail in Appendix~\ref{apx:thm2}. Furthermore, if one wishes to extend the notion of pairwise repeatability, as defined in Definition~\ref{def:beta-r}, to encompass $k$ independent trials, it suffices to adjust the power coefficient in~\eqref{eq:beta-2} (e.g., replacing $(1 - c)^2$ with $(1 - c)^k$ to reflect $k$-way repeatability across independent executions). However, this generalization is often unnecessary in the standardized testing context or applications of interest by this paper. In most cases, repeatability is assessed by asking whether a single independent stakeholder (i.e., a \emph{replicator}) can replicate the outcome produced by another (i.e., an \emph{initiator})—i.e., a pairwise setting. Even when multiple stakeholders attempt to repeat a result, the comparison typically remains pairwise between the original test and each independent repetition.

Moreover, the properties established in Theorems~\ref{thm:quant-acc} and~\ref{thm:quant-rep} hold for any almost uniform $\alpha$-partition of the output space of the testing algorithm, provided that the partition remains fixed across repeated trials. This fixed partition, along with other hyperparameters such as $c$, $\beta$, and $\gamma$, collectively serves as a ``random-seed"-like configuration that enables repeatable testing in robotics settings involving uncertainties with software, hardware, environment, and combinations of the factors.

By combining Theorems~\ref{thm:quant-acc} and~\ref{thm:quant-rep}, we obtain the following guarantee: with probability at least $1 - c$ (for $c \in (0,1)$), the output of the testing algorithm deviates from the ground-truth estimate by no more than $\gamma + \frac{\alpha}{2}$ ($\gamma \geq 0$ is given and $\alpha$ determined by \eqref{eq:alpha}), and with probability at least $1 - \beta$ (for $\beta \in (0,1)$), repeated executions of the algorithm return exactly the same quantized estimate. 
\begin{remark}\label{rmk:acc-rep}
    Note that the condition $(1 - c)^2 \geq (1 - \beta)$ is crucial for the validity of~\eqref{eq:alpha}. More importantly, it underscores a fundamental intrinsic relationship between accuracy and repeatability: one cannot achieve higher confidence in pairwise repeatability than in pairwise accuracy. In other words, the confidence level for repeatability is inherently bounded by that of the underlying accuracy guarantee. This can also be interpreted from a stability perspective, as we shall see in the following section: less stable algorithms are inherently less accurate and, consequently, less repeatable.
\end{remark}

We now return to justify the earlier assumption of $\gamma$-accuracy.

\subsection{Testing accuracy}
For start, the update~\eqref{eq:is} can be re-written as
\begin{equation}\label{eq:is-dtsd}
    r_{n} = r_{n-1} + \frac{1}{n} \left(\psi(x_n)\frac{p(x_n)}{q(x_n)} - r_{n-1} \right)
\end{equation}
The $\gamma$-accuracy of the testing algorithm is established by demonstrating that the above estimator converges to $r^*$ by \eqref{eq:rstar} almost surely.

While there are many ways to showcase the accuracy property, including but not limited to techniques involving Martingale theory~\cite{tokdar2010importance}, Robbins–Monro stochastic approximation theory, central limit theorems~\cite{feng2020testing}, our analysis for this paper emphasizes an alternative route. Specifically, we show that the desired result can be achieved through a direct application of Lyapunov's theorem on the stability of discrete-time stochastic dynamical systems (see Theorem~\ref{thm:lyapunov} in Appendix~\ref{apx:lya}). A straightforward extension of the proof in Appendix~\ref{apx:lya} establishes that the estimator $r_n$ converges almost surely to the true value $r^*$ with the Lyapunov function $V(r_n)= (r_n-r^*)^2$. From this convergence result, a probabilistic bound (through Markov's inequality) can then be derived for any specified tolerance level $\gamma$, as
\begin{equation}\label{eq:lya-conv}
    \mathbb{P}(|r_n-r^*| \geq \gamma) \leq \frac{\mathbb{E}[V(r_n)]}{\gamma^2}.
\end{equation}
Notably, this stability-based analysis imposes minimal—if any—distributional assumptions. It does not require the samples to be identically distributed, nor does it rely on Gaussianity, Markovian structure, stationarity, or other common statistical conditions. This generality underscores the flexibility of the proposed algorithm and its suitability for real-world standard testing practices, where such idealized assumptions often do not hold.

Specifically, by removing the requirement of identically distributed samples, the importance distribution $q$—or more generally, the specific samples or test cases—can be defined adaptively in a variety of ways. This flexibility allows for the use of adaptive sampling techniques and supports a broad range of feedback testing policies, including those parameterized by neural networks. All such approaches remain fully compatible with the proposed framework.

The inequality in~\eqref{eq:lya-conv} also reveals a variance-aware bound, which naturally emerges when one explicitly accounts for incremental and conditionally stochastic behavior. However, in most practical applications—including those using non-Lyapunov-based approaches—variance-based bounds are often overly conservative or loosely approximated. A prime example is Hoeffding's inequality~\cite{hoeffding1994probability,mcdiarmid1998concentration}, which assumes uniform bounds on the random variable and does not adapt to observed sample variance (as illustrated later in Eq.~\eqref{eq:hoeffding}). Such bounds fail to capture the finer structure of sample-wise variability and may significantly overestimate the required testing effort.

While the Lyapunov-based approach offers a more flexible theoretical framework and can, in principle, incorporate recursive updates and second-moment information through drift conditions, it is not optimized for producing tight, sample-efficient probabilistic bounds in practice. This limitation motivates the final component of our analysis: an empirical Bernstein inequality–based method~\cite{maurer2009empirical}, which provides sharper, variance-sensitive bounds that adapt to observed variability, resulting in greater testing efficiency while preserving accuracy and repeatability guarantees.

\subsection{Testing efficiency}
Efficient testing relies on triggering termination as early as possible (see Line 2 of Algorithm~\ref{alg:is_testing}). The proposed termination condition is formally characterized by the following theorem.
\begin{theorem}\label{thm:var_rp}
    The $\alpha$-quantized (with $\alpha$ determined by~\eqref{eq:alpha} given $c,\beta\in(0,1)$ and $\gamma$) Algorithm~\ref{alg:is_testing} with termination condition
    \begin{equation}
        \mathcal{T}(T) = \begin{cases}
            1  & \text{ if $ n = |T|$ satisfies \eqref{eq:n_bar}} \\
            0 & \text{ otherwise}
        \end{cases}
    \end{equation}
    \begin{equation}\label{eq:n_bar}
        \gamma \geq \sqrt{\frac{2\hat{\sigma}_n \ln(2/c)}{n}} + \frac{7(m\overline{w})^2\ln(2/c)}{3(n-1)}.
    \end{equation}
    \begin{equation}
        \hat{\sigma}_n = \frac{1}{|T|}\sum_{i=1}^{|T|} (\psi(x_i)w(x_i) - r_n)^2
    \end{equation}
    is $\gamma+\frac{\alpha}{2}$-accurate (with probability $1-c$ for $c\in(0,1)$) and $\beta$-repeatable (with $\alpha$ determined by~\eqref{eq:alpha} supplied with $\gamma, c$, and $\beta\in(0,1)$).
\end{theorem}
\begin{proof}
    To prove the given bound~\eqref{eq:n_bar}, A two-sided tail bound is first derived (with some minor modifications scaling the original bound $[0,1]$ in~\cite{maurer2009empirical} to $[0,m\overline{w}]$ to accommodate the variance) from the empirical Bernstein inequality (Theorem 4~\cite{maurer2009empirical}, and Theorem 11~\cite{maurer2009empirical} for independent (but not identically distributed) samples), which leads to
    \begin{equation}\label{eq:direct_ebe}
        |r_n-r^*| \leq \sqrt{\frac{2\hat{\sigma}_n \ln(2/c)}{n}} + \frac{7(m\overline{w})^2\ln(2/c)}{3(n-1)}.
    \end{equation}
    Having $\gamma$ $\geq$ R.H.S. of~\eqref{eq:direct_ebe} leads to the suggested condition \eqref{eq:n_bar}. That is, with $n$ sampled tests satisfying the inequality of~\eqref{eq:n_bar}, Algorithm~\ref{alg:is_testing} is $\gamma$-accurate (with probability $1-c$). By Theorem~\ref{thm:quant-acc} and Theorem~\ref{thm:quant-rep}, the $\alpha$-quantized Algorithm~\ref{alg:is_testing} is thus $\gamma+\frac{\alpha}{2}$-accurate and $\beta$-repeatable. 
\end{proof}
Given the above, $m$ and $\overline{w}$ are separately considered to ensure that the worst-case scenario corresponds to the case where the sample with the largest measure also receives the largest weight. In practical testing scenarios with structural insights or empirical evidence, one can adopt a smaller empirical or conditional bound on $\psi(\cdot)\omega(\cdot)$ rather than separately bounding the two terms.



Note the empirical Bernstein inequality implied bound in~\eqref{eq:direct_ebe} explicitly utilize the online estimate of variance. When the empirical variance $\hat{\sigma}_n \ll (m\overline{w})^2$, this bound decays faster than Hoeffding's, leading to a tighter confidence with fewer samples needed. In the worst case (variance near max), it’s within a constant factor of Hoeffding’s as
\begin{equation}\label{eq:hoeffding}
    \mathbb{P}(|r_n-r^*| \geq \gamma) \leq 2\exp{\Bigg( -\frac{2n\gamma^2}{(m\overline{w})^2} \Bigg)}
\end{equation}
In practice, the termination condition can incorporate both bounds, i.e., choosing whichever yields the tighter criterion. Consequently, the practical efficiency of the testing algorithm hinges on the empirical variance $\hat{\sigma}_n$, which is itself influenced by the choice of the importance distribution $q$. This connects to a broad class of strategies in accelerated, adaptive, and even adversarial testing~\cite{feng2020testing,feng2021intelligent,zhao2016accelerated,li2020av,capito2020modeled,lee2019adastress,koren2018adaptive}. Regardless of the specific approach used to define, affect, or implicitly impact $q$, the proposed method enables testing that is accurate, repeatable, and sample-efficient.

As previously discussed, the variance-dependent bound presented in this paper is not the only option available, nor is it unique in accommodating non-identically distributed samples. The Lyapunov-based approach introduced earlier is also well-suited for this purpose and offers a flexible theoretical framework. However, to the best of our knowledge, the bound given in Eq.~\eqref{eq:direct_ebe}, derived from the empirical Bernstein inequality, consistently yields the tightest results in practice. Among the methods we have explored, it enables the most sample-efficient execution of repeatable SQ-type testing algorithms discovered to date. It also notably surpasses existing theoretical results, including the KL-divergence-parameterized bound by Capito et al.\cite{capito2024repeatable}, as well as the bound for Monte Carlo sampling-based SQ by Impagliazzo et al.\cite{impagliazzo2022reproducibility} (primarily based on the Hoeffding's~\cite{hoeffding1994probability} as shown by \eqref{eq:hoeffding}).

In summary, the practical execution of the proposed method, between a \emph{testing initiator} and any \emph{replicator}, is outlined below, with guarantees of repeatability, accuracy, and efficiency. 
\begin{algorithm}
    \begin{algorithmic}[1]
    \State {\bf 1. Testing Initiator}: 
    \State {\ \ \bf Given $X$, $p$, $q$, $f$, $c$, $\gamma$, $\beta$, $\mathcal{M}$, $m, \bar{\omega}$}
    \State {\ \ Decide $\alpha$ by~\eqref{eq:alpha} and $\alpha$-partition of $\mathcal{M}$ by Definition~\ref{def:alpha-part}}
    \State {\ \ {\bf Output:} $\alpha$-quantized Algorithm~\ref{alg:is_testing} by Theorem~\ref{thm:var_rp}}
    \State {\bf 2. Testing Replicator}:
    \State {\ \ {\bf Given} $X$, $p$, $q$, $f$, $c$, $\gamma$, $\beta$, $\mathcal{M}$, $m, \bar{\omega}$, $\alpha$-partition of $\mathcal{M}$}
    \State {\ \ {\bf Output:} $\alpha$-quantized Algorithm~\ref{alg:is_testing} by Theorem~\ref{thm:var_rp}}
    \end{algorithmic}
    \caption{Summary of the proposed algorithm} \label{alg:ris_testing}
\end{algorithm}

We conclude this section with some comments on hyperparameter specifications. For parameters like $c$, $\beta$, and $\gamma$, their interpretations are standard and intuitive. We additionally note that unlike the confidence level coefficients, $\gamma$ carries physical meaning and may have different units depending on the application context, but this does not affect the validity of our theoretical bounds. Furthermore, as established in Theorem~\ref{thm:quant-rep} and discussed in Remark~\ref{rmk:acc-rep}, there exists a fundamental trade-off between accuracy and repeatability: the parameters $c$ and $\beta$ must satisfy certain relative conditions for the framework to be well-defined, reflecting the inherent tension between achieving high confidence in accuracy (small $c$) and high repeatability (small $\beta$). The bounds $m$ and $\overline{\omega}$ also play important roles in our framework through \eqref{eq:direct_ebe} and other related bounds, and should be determined adaptively based on the specific problem characteristics (see Section~\ref{sec:exp} for practical examples). Finally, $\alpha$ is not a directly specifiable parameter and must be determined indirectly through $c$, $\beta$, and $\gamma$.

\section{Experiments \& Applications}\label{sec:exp}
The results presented in this section extend beyond a mere empirical validation of the theoretical findings proven in Section~\ref{sec:main}. They selectively highlight practical applications by: (i) demonstrating how the proposed method can integrate with existing standards to enhance the interpretability of outcomes and generalize to broader classes of robotic systems; (ii) making ongoing standardization efforts practically feasible through reduced testing effort and improved repeatability and accuracy; and (iii) laying a practical foundation for robot evaluation in domains lacking established formal testing standards.

The applications presented span diverse robotic domains, including manipulators, automated vehicles, and humanoid robots. They address a wide variety of performance measures, such as positioning repeatability, collision-avoidance risk assessment, and trajectory-tracking accuracy. Moreover, the testing scenarios encompass simulation-based evaluations, real-world experiments, and hybrid settings combining both approaches.

\begin{remark}
    The primary focus of this study lies in the testing perspective and its associated properties. While details about the subject robots and their control algorithms are necessary for experimental completeness, they are not central to the testing methodology and can, in principle, be replaced with alternative systems. Consequently, certain implementation details—such as the impedance-based standing controller used for the humanoid robot—may be omitted. Additional configuration information, where relevant, can be found in corresponding open-source resources~\cite{nadegithub,unitreegithub}, if available.
\end{remark}

\begin{remark}
    In alignment with Remark~\ref{rmk:rp}, the experimental results presented in this section may incidentally highlight several issues and challenges associated with existing robotic systems. For instance, Section~\ref{sec:exp-manip}, to the best of our knowledge, offers the first empirical study on manipulation positioning repeatability in humanoid robots operating under whole-body manipulation. The findings may indicate significant limitations that warrant further attention from the robotics community. However, these insights are not the primary focus (nor the intended objective) of this paper. Likewise, Section~\ref{sec:exp-loco} may be interpreted as a critique of the locomotive accuracy of certain practical implementations. While these observations are outside the formal scope of this work, they nevertheless represent promising directions for future investigation.
\end{remark}

The experiments and applications presented include some comparisons with algorithms based on statistical significance criteria; however, comparisons with alternative methods that offer formal repeatability guarantees are limited, as such studies remain extremely scarce. Nevertheless, the advantages of the proposed algorithm—made explicit through Theorem~\ref{thm:var_rp}—are readily apparent when compared to prior approaches, such as those by Capito et al.\cite{capito2024repeatable} and Impagliazzo et al.\cite{impagliazzo2022reproducibility}, particularly in the context of importance sampling and Monte Carlo sampling routines.

The following observations are highlighted before proceeding to individual applications:
\begin{itemize}
    \item \textbf{Repeatability} is consistently demonstrated across all empirical examples, aligning with the theoretical guarantees established in Theorem~\ref{thm:quant-rep} and related results.
    \item \textbf{Accuracy} is generally more challenging to evaluate empirically, as the testing scenarios explored in this section involve complex and partially unobservable systems. However, in special cases—such as the NADE study in Section~\ref{sec:exp-nade}, approximations of the ground-truth value $r^*$ with sufficiently high confidence are available. In those cases, the output of our proposed algorithm remains bounded and consistent with the accuracy guarantees provided by Theorem~\ref{thm:quant-acc}.
    \item \textbf{Efficiency} is particularly evident in empirical results where variance reduction is applicable. Depending on the utility of the specific adaptive testing algorithm and the nature of the underlying problem (i.e., whether the testing process can be fundamentally accelerated), the absolute gains may vary. Nonetheless, relative improvements over traditional baselines remain consistently observable.
\end{itemize}

\subsection{Manipulation positioning tests}\label{sec:exp-manip}
\noindent\textbf{Testing configurations \& execution}: Recall the results shown in Fig.~\ref{fig:iso9283_np}(b), a Unitree G1 humanoid robot is equipped with a whole-body impedance controller to maintain stable and robust lower-body behavior (to a certain extent), along with a standard inverse kinematics solution for upper-body manipulation, including yaw control at the waist. In accordance with the procedures recommended by ISO 9283~\cite{iso9283}, an initial positioning point near the robot’s torso is established and designated as the reference position for assessing positioning repeatability. The robot repeatedly performs a standardized two-step cycle until the required sample size for statistical evaluation is achieved. The steps include: (i) moving the end-effector to a specified target position, with target locations uniformly distributed within a predefined test workspace situated to the right of the reference position. Each motion segment consistently approaches the measurement location from the same direction (right to left), and the robot maintains a stationary state at this location for a minimum dwell time of three seconds; (ii) returning precisely to the reference position and again maintaining a stationary state for a minimum dwell period of three seconds before proceeding with the next cycle.

ISO 9283, primarily developed for evaluating industrial manipulators, recommends collecting between 20 and 50 samples for each reference point. Due to the significant variability observed in our experimental results, we have substantially increased the number of samples: to 250 samples in Fig.\ref{fig:iso9283_np}, and further up to 700 samples as shown in Fig.\ref{fig:g1_manip_compare}. Additionally, while the complete ISO 9283 standard typically involves multiple (usually five) distinct reference positions for comprehensive evaluation, the testing procedure employed in this study has simplified the protocol by focusing solely on a single reference position.

In Fig.~\ref{fig:iso9283_np}, we present three independent executions of the testing procedure for evaluating the manipulation repeatability of the G1 humanoid robot, showcasing the full distribution of end-effector displacement for each trial. Fig.~\ref{fig:g1_manip_compare} extends this evaluation by executing the same procedure over 200 independent trials, with each trial consisting of 700 samples as previously described. For each trial, we record the average end-effector displacement (in millimeters) and the corresponding standard deviation. The distribution of the 200 average displacement values is shown in green in Fig.~\ref{fig:g1_manip_compare}(a), while the distribution of their associated deviations is shown in green in Fig.~\ref{fig:g1_manip_compare}(b).

Overall, Fig.~\ref{fig:g1_manip_compare} reinforces the issue highlighted in Fig.~\ref{fig:iso9283_np}: the standard ISO-9283 procedure, as applied to this setting, fails to yield consistent and repeatable outcomes across independent executions. 

To further demonstrate the effectiveness of our proposed method, we set the parameters as follows: $\gamma = 0.1$, $c = 0.05$, $\beta = 0.1$, $m = 6$, and $\bar{\omega} = 1$. Although the displacement in this task is theoretically unbounded, it remains tightly constrained in practice. The choice of $m = 6$ mm serves as a conservative upper bound, substantially exceeding the maximum displacement observed empirically, and it successfully upper-bounds all displacement values recorded across trials. The parameter $\bar{\omega} = 1$ represents the theoretical worst-case under uniform sampling assumptions. According to Theorem~\ref{thm:var_rp}, this setup ensures that the expected displacement lies within 0.195 mm of the (primarily unknown) ground truth displacement value with at least 95\% probability, and that any pair of repeated trials yields the same outcome with at least 90\% probability.

\noindent\textbf{Results \& discussions}: Empirically, in real-world testing, we first executed the testing initiator's procedure from Algorithm~\ref{alg:ris_testing} to establish a baseline reference, followed by 200 independent executions of the testing replicator's procedure. The distribution of expected displacement values across these 200 replicated trials is shown in red in Fig.~\ref{fig:g1_manip_compare}(a). The algorithm achieved a 98\% repeatability rate, with only 4 exceptions out of 200 trials. The corresponding fixed error bound, $\gamma + \frac{\alpha}{2} = 0.195$, holds uniformly across trials and is depicted in Fig.~\ref{fig:g1_manip_compare}(b) in red. 

Compared to the standard procedure adapted from ISO~9283, the proposed algorithm demonstrates significantly higher repeatability while providing more direct and informative insights into performance. Due to small random variations, the number of samples collected across the 200 trials was not fixed, but consistently fell within the limited range of 598 to 601—nearly 100 fewer samples per trial than adopted by the ISO 9283 adaptation. It is important to note, however, that the proposed method does not inherently make the testing process more \emph{efficient} in terms of sample reduction. A closer analysis reveals that the displacement error is nearly uniformly distributed with respect to the initial position across the configuration space described in step (i) of the testing procedure. In other words, no particular choice of $q$ can significantly accelerate convergence or improve sample efficiency under this setting.

Nevertheless, when faced with testing outcomes that are noisy and potentially difficult to interpret, the proposed method is able to \emph{interpret} and structure these results in a way that enables statistically meaningful, repeatable, and reliable evaluation, something the standard ISO procedure struggles to ensure under similar noise conditions.

\begin{figure}
    \vspace{1mm}
    \centering
    \includegraphics[trim={3cm 0cm 2.5cm 1cm},clip,width=0.99\linewidth]{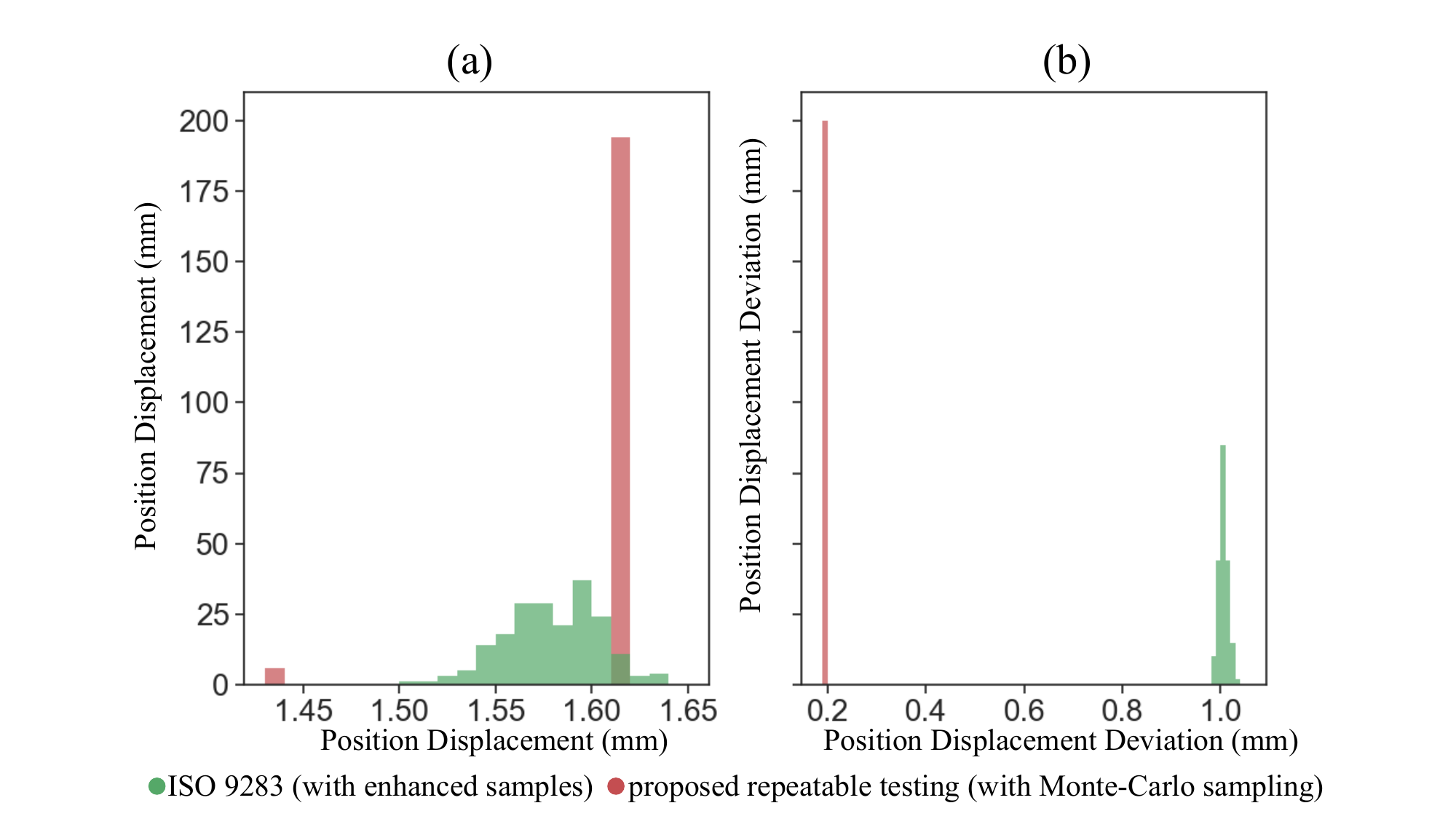}
    \caption{A comparative study is conducted to evaluate the repeatability of two testing methods through manipulation repeatability assessment of the Unitree G1 humanoid robot. Both methods are based on Monte Carlo sampling but differ in design. The first follows a procedure adapted from ISO 9283 with an extended number of samples (700), while the second employs the proposed Algorithm~\ref{alg:ris_testing}, which incorporates a structured quantization mechanism to enhance repeatability. Subfigures (a) and (b) show the histogram distributions of position displacement and deviation, respectively.}
    \label{fig:g1_manip_compare}
    \vspace{-3 mm}
\end{figure}

\subsection{Intelligent driving collision avoidance tests}\label{sec:exp-nade}

\begin{figure*}
    \vspace{1mm}
    \centering
    \includegraphics[trim={0cm 2cm 0cm .5cm},clip,width=0.95\linewidth]{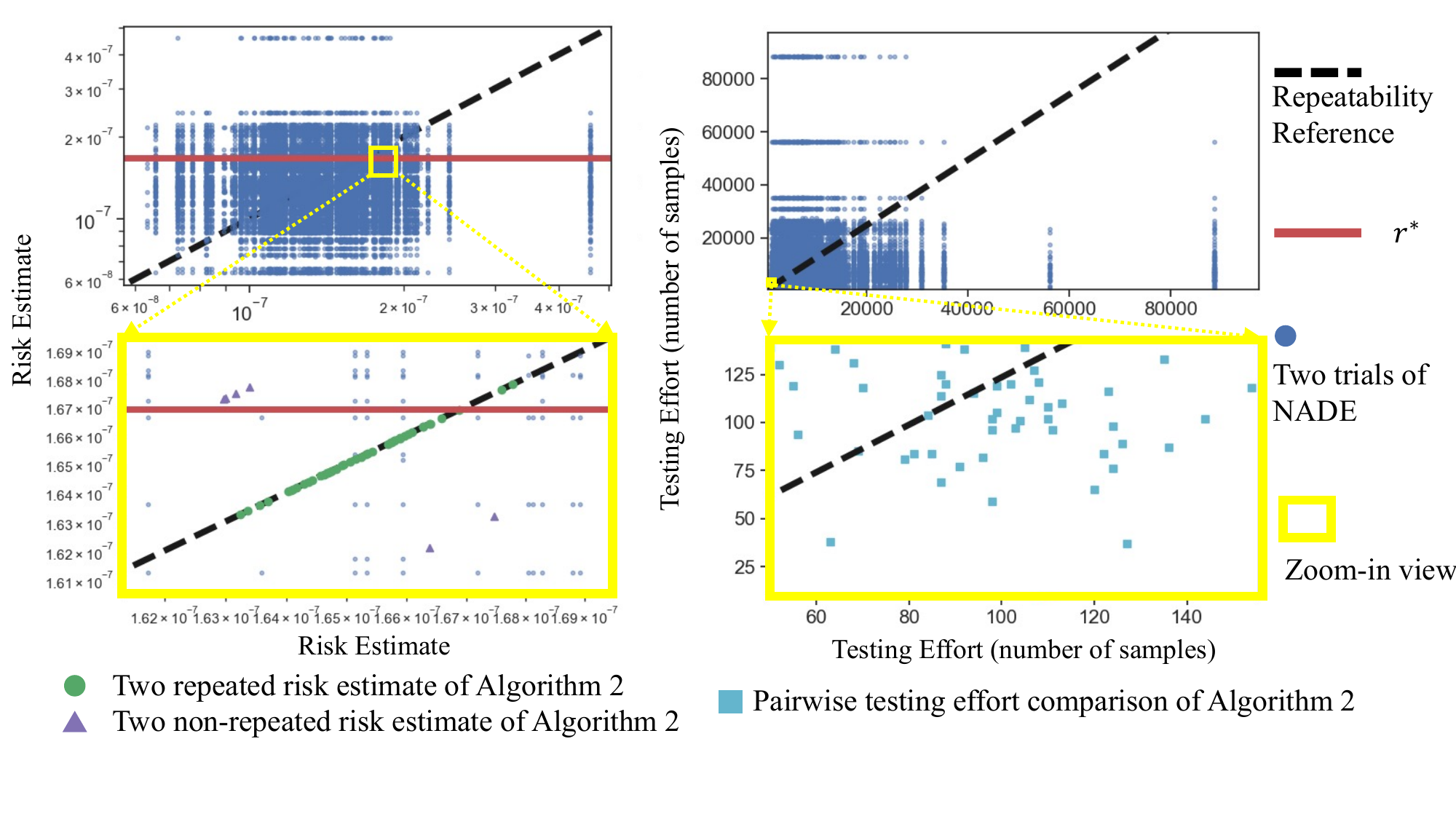}
    \caption{Integrating the proposed method with NADE previously shown in Fig.~\ref{fig:nade_np} with the non-repeatability issue.}
    \label{fig:nade}
    \vspace{-4mm}
\end{figure*}
This case study integrates the proposed algorithm for repeatability and accuracy evaluation into an already accelerated algorithm, NADE~\cite{feng2021intelligent}, within the context of collision risk estimate of automated vehicles. A similar issue was identified by Capito et al.~\cite{capito2024repeatable}, but their approach did not demonstrate a resolution within the NADE framework, largely due to the lack of publicly available information regarding the KL-divergence between distributions $p$ and $q$ in the original NADE open-source code~\cite{nadegithub}. Even if such information were available, their KL divergence–based methodology lacks the efficiency and effectiveness of the proposed variance-dependent bound established in Theorem~\ref{thm:var_rp}.

\noindent\textbf{Testing configurations \& execution}: In this study, we implemented an $\alpha$-quantized version of NADE. Our method naturally ``wraps" around the original NADE algorithm and augments it with structured guarantees of repeatability and accuracy—on top of the sampling efficiency and variance reduction already achieved through NADE’s importance sampling design.

Recall the NADE test results presented in Fig.~\ref{fig:nade_np}. The same results are reorganized as the first-row plots in Fig.~\ref{fig:nade} for consistency with subsequent comparisons. Let the parameters be set as follows: $\gamma = 3 \times 10^{-9}$, $c = 0.01$, $\beta = 0.1$, and $m\bar{\omega} = 0.0001$. Notably, the empirical upper bound of the term $m\bar{\omega}$ observed across all experiments was approximately $3.4 \times 10^{-7}$, indicating that the chosen bound is significantly more conservative and thus more than sufficient for the intended guarantees. According to Theorems~\ref{thm:quant-acc} and~\ref{thm:quant-rep}, the resulting tolerance for the accuracy of the $\alpha$-quantized NADE is approximately $5.14 \times 10^{-9}$.

In this study, we repeatedly executed Algorithm~\ref{alg:ris_testing}, in this case, the $\alpha$-quantized NADE—with one initiator and one replicator over 100 independent pairwise trials. Each pair produced two risk estimates along with their corresponding testing efforts (i.e., the number of tests performed). The results are presented in Fig.~\ref{fig:nade}. Due to the extremely concentrated outcomes produced by the proposed method, the data points are not clearly visible in the first-row plots of Fig.~\ref{fig:nade} (which are identical to Fig.~\ref{fig:nade_np} without the proposed $\alpha$-quantization); therefore, zoomed-in views are provided in the corresponding subfigures in the second row. For the risk estimates, we further distinguish between repeatable and non-repeatable pairwise trials—depicted as green dots and purple triangles, respectively. For testing effort comparisons, this distinction is not applied. The ground truth estimate of collision risk, originally reported by Feng et al.~\cite{feng2021intelligent} and its accompanying codebase~\cite{nadegithub}, based on extensive testing in the Naturalistic Driving Environment (NDE), is also highlighted as a red segment.

\begin{figure}
    \centering
    \includegraphics[trim={2.5cm 1cm 5cm 4cm},clip,width=0.99\linewidth]{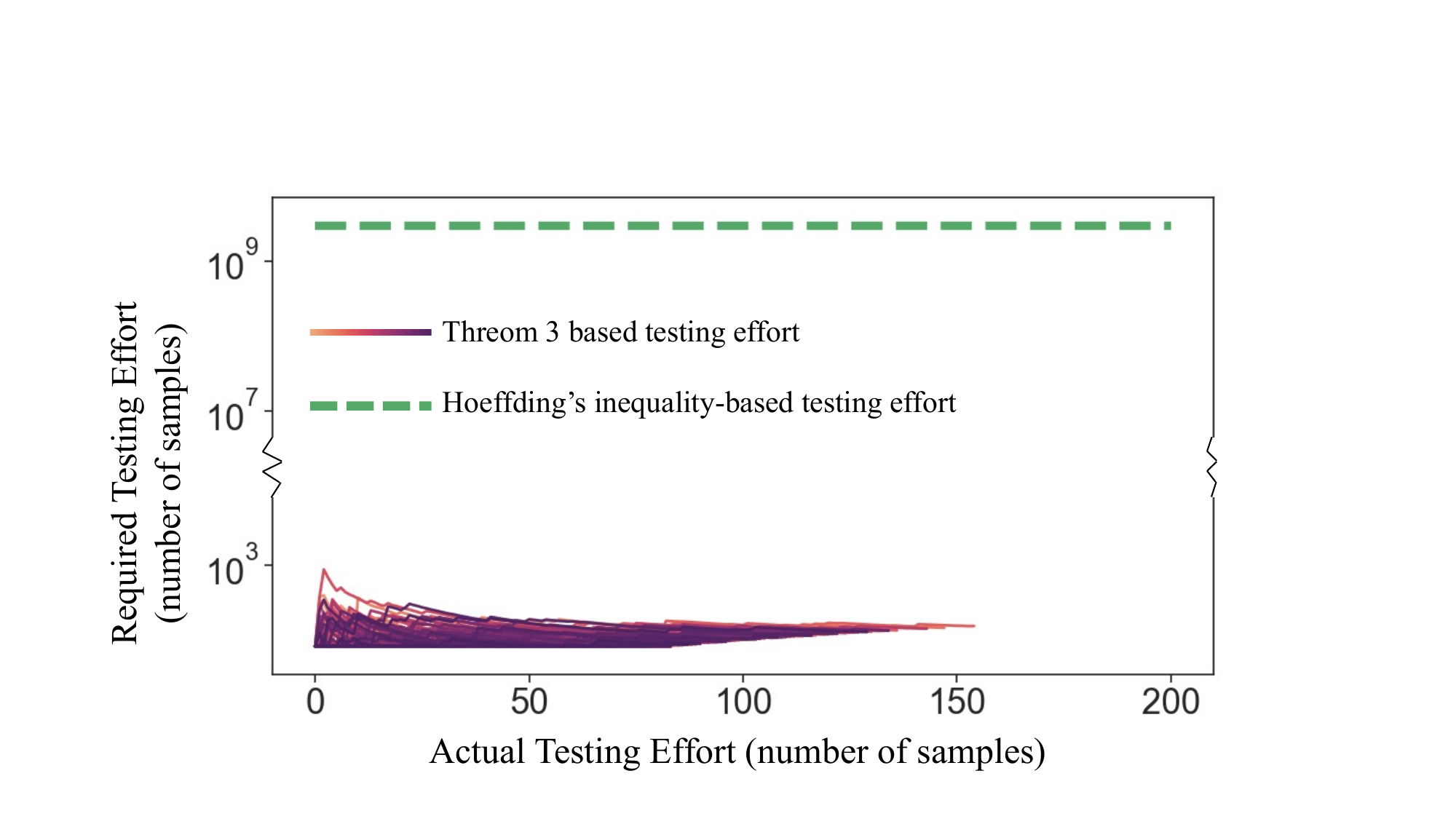}
    \caption{Comparing the required testing effort \eqref{eq:n_bar} by Theorem~\ref{thm:var_rp} (across 100 independent trials) with the one derived from the Hoeffding's inequality of~\eqref{eq:hoeffding}.}
    \label{fig:nade_n}
    \vspace{-7mm}
\end{figure}

\noindent\textbf{Results \& discussions}: Overall, the improvement of the $\alpha$-quantized NADE over the original formulation is immediate and substantial. Risk estimates are tightly clustered around the ground truth value, all falling within the theoretical accuracy bound of $5.14 \times 10^{-9}$. Only 5 out of 100 pairwise trials resulted in non-repeatable estimates, aligning closely with the theoretical bound set by $\beta = 0.1$, which guarantees at least 90\% repeatability.

More importantly, the performance gains from variance reduction and the associated acceleration of testing as originally claimed by the NADE authors are fully leveraged in our framework via Theorem~\ref{thm:var_rp}. The most sample-intensive trial under our proposed method required approximately 600 times fewer samples than the worst-case original NADE trial, and about 80 times fewer than the average testing effort observed in those trials (and about 1 million times fewer than the testing effort claimed to get the $r^*$ approximation in NDE by Feng et al.~\cite{feng2020testing,nadegithub}). Moreover, as illustrated in Fig.~\ref{fig:nade_n}, the testing effort required under Theorem~\ref{thm:var_rp}—which leverages the empirical Bernstein inequality and adapts to online variance estimates—clearly benefits from the substantial variance reduction enabled by NADE. At comparable stages of testing progress, the sample requirement under the proposed method is consistently over one million times smaller than that prescribed by Hoeffding's inequality when using a constant, conservative over-approximation of variance.

Regarding accuracy, although the relative half-width (RHW) criterion used in the original NADE study is not directly comparable to the $\gamma$-based accuracy adopted in our method, some meaningful parallels can be drawn. The original trials employed an RHW threshold of 0.03, which corresponds to an absolute tolerance of approximately $5 \times 10^{-9}$ based on the average risk estimate, with a 95\% confidence level. In contrast, our method achieves a similar level of accuracy ($\gamma + \frac{\alpha}{2} \approx 5.14 \times 10^{-9}$), but with a higher confidence level of 99\%, drastically fewer samples, and significantly improved repeatability.

Finally, we note that both the original NADE and our proposed $\alpha$-quantized modification were executed in simulation. However, given NADE’s configuration, where each test involves approximately 400 meters of subject vehicle travel, and the substantial efficiency gains achieved through our method, the resulting test effort becomes nearly comparable to that of real-vehicle evaluations conducted under existing standardized procedures for Automated Driving Assistance Systems (ADAS). For example, the EURO New Car Assessment Program (NCAP) AEB Car-to-Car Test Protocol~\cite{euroncap2023aeb} specifies 134 test scenarios for validating Crash Imminent Braking (CIB) and Automatic Emergency Braking (AEB) systems, with each scenario involving a comparable travel distance to that used in NADE. 

It is important to highlight that the standard procedures mentioned above focus on much simpler subsystems—specific ADAS features such as AEB, whereas NADE targets a significantly more advanced level of vehicle autonomy. This suggests that our proposed approach has the potential to elevate NADE into a practical framework for standardized, regulatory-grade evaluation of autonomous driving systems. By significantly enhancing its repeatability and accuracy—while preserving its sampling efficiency—our $\alpha$-quantized extension moves NADE closer to being a viable candidate for formal adoption in safety validation pipelines, akin to those used for ADAS features under programs such as EURO NCAP and others~\cite{rao2019tja}. While fully adopting such standardized protocols represents a valuable direction for future research, potentially involving not only algorithmic advancements but also enhancements to testing apparatus and experimental procedures, it is beyond the scope of the current study.

\subsection{Humanoid robot locomotive tracking performance tests}\label{sec:exp-loco}
In this final case study, we consider a testing scenario that, to the best of our knowledge, currently lacks established testing standards such as ISO 9283, as well as widely recognized research benchmarks like NADE. This scenario reflects a common situation encountered during the early stages of developing standardized tests, characterized by limited prior knowledge of the system, sparse data availability, and incomplete procedural frameworks. Additionally, we use this case to highlight a key advantage of our proposed method—its ability to support the utilization of non-identical testing algorithms. Specifically, our framework accommodates scenarios involving adaptive algorithms or heterogeneous implementations, enabling repeatable evaluation even when the algorithms under comparison are not strictly identical. 

\noindent\textbf{Testing configurations \& execution}: Consider a Unitree G1 humanoid robot equipped with a locomotion policy trained via reinforcement learning (RL), using the implementation provided in~\cite{unitreegithub}. The policy accepts commanded linear velocities defined at the center of pelvis in the forward-backward ($x$) and lateral ($y$) directions, along with a desired yaw rate, forming the input command vector $\s^d = [v_x^d, v_y^d, \theta^d]$. In this study, the admissible command space is parameterized as $v_x^d \in [-0.3, 0.3]$ m/s, $v_y^d \in [-0.3, 0.3]$ m/s, and $\theta^d \in [-0.3, 0.3]$ rad/s.

However, due to well-documented challenges associated with odometry on legged robotic platforms~\cite{Hartley-RSS-18,yang2023multi}, accurate tracking of commanded inputs using proprioceptive sensors alone (i.e., without vision-based sensing) is frequently imperfect. Specifically for this study, the locomotion policy under evaluation, following the setup from~\cite{unitreegithub}, does not use any odometry information (e.g., linear position or velocity of the pelvis) as part of its input. During training in simulation via a RL framework, such privileged state information is provided only to the reward function to encourage the policy to follow commanded states. However, tracking behavior is never explicitly reinforced through direct feedback mechanisms. This design choice results in various deficiencies, as also discussed in~\cite{khor2025post}. Moreover, the robot's tracking performance is not uniformly consistent across the command space. It is intuitively expected that commands of larger magnitude will generally result in lower tracking accuracy. However, it remains unclear whether this correlation is linear or to what extent the monotonic relationship holds.

In each test, the robot is initialized with zero velocity commands and remains in-place locomotion in a dwell state for at least 3 seconds. Subsequently, a randomly sampled command $\s^d$, drawn from the admissible set, is applied and maintained constant for the remainder of the test. A trajectory is then recorded over a minimum duration of 5 seconds at 50 Hz. The final 3 seconds of this data (i.e., the last 150 time steps) are extracted as the stabilized observed state trajectory, denoted $T^o =\{\s^o_i\}_{i=1}^{150}$.

To evaluate tracking performance, we compare the observed state trajectory $T^o$ with the commanded trajectory $T^d$ (which consists of repeated instances of $\s^d$), using a soft L-infinity exponential loss defined as:
\begin{equation}\label{eq:l-inf-loss}
\psi(T^o; T^d) = 1 - \exp\left(-6 \sum_{i=1}^{150} \|\s^o_i - \s^d\|^2\right),
\end{equation}
which yields a bounded error in the interval $[0,1]$. Note that the formulation in \eqref{eq:l-inf-loss} is specifically adopted and parameterized to emphasize larger errors, avoiding excessive sensitivity to smaller errors through its bounded nature. We found this particularly appropriate for legged locomotion tracking tasks, as minor errors are inherently present due to the periodic stepping patterns and the articulated, limb-based structure characteristic of humanoid robots and most legged robotic systems. 

In any standardized testing context, this distribution should be selected deliberately and justified clearly. When detailed operational domain knowledge is available (such as extensive empirical data on human driver behavior), one can leverage such prior information to specify an informed target distribution (as exemplified by NADE~\cite{feng2021intelligent}, which selects NDE as references to construct both the target distribution and the performance reference $r^*$). Unfortunately, this is not applicable in the current example due to limited domain-specific data. Therefore, identifying a more representative and practically meaningful distribution remains an important direction for future research.

\begin{figure}
    \vspace{1mm}
    \centering
    \includegraphics[trim={1cm 4cm 14cm 3cm}, clip,width=0.9\linewidth]{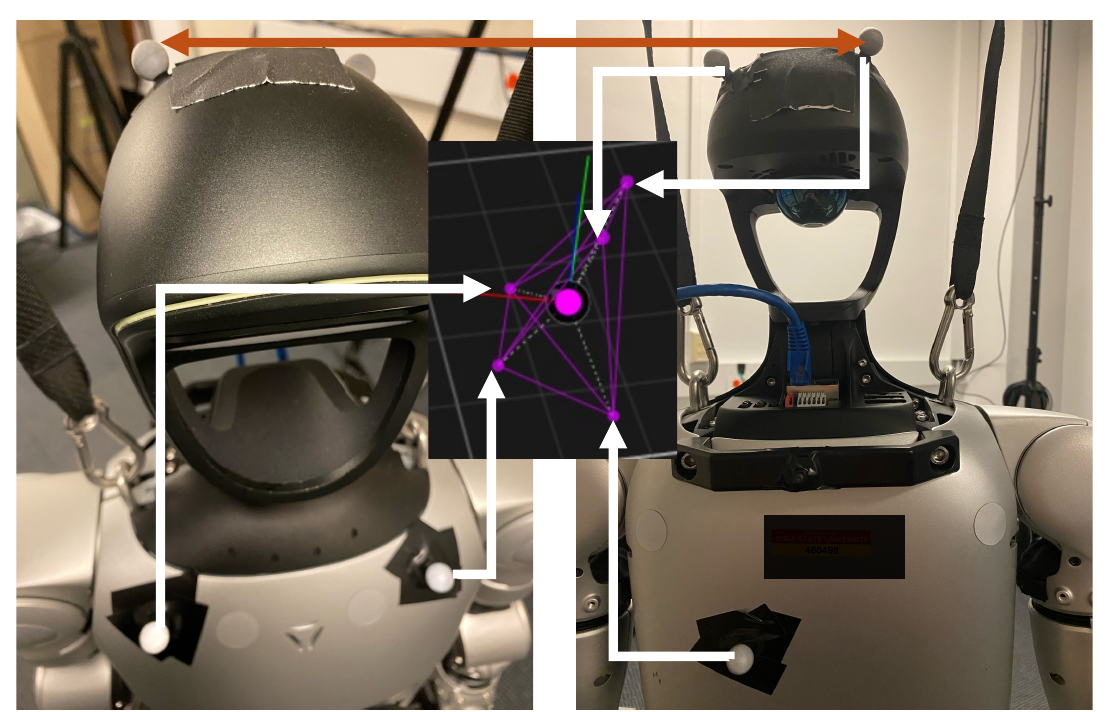}
    \caption{The tracker installation on the G1 robot used for the real-world evaluation of tracking performance.}
    \label{fig:g1_trackers}
    \vspace{-5mm}
\end{figure}

\begin{figure*}
    \vspace{2mm}
    \centering
    \begin{subfigure}{.49\textwidth}
      \centering
      \includegraphics[trim={3cm 0cm 4cm 2cm}, clip, width=0.99\textwidth]{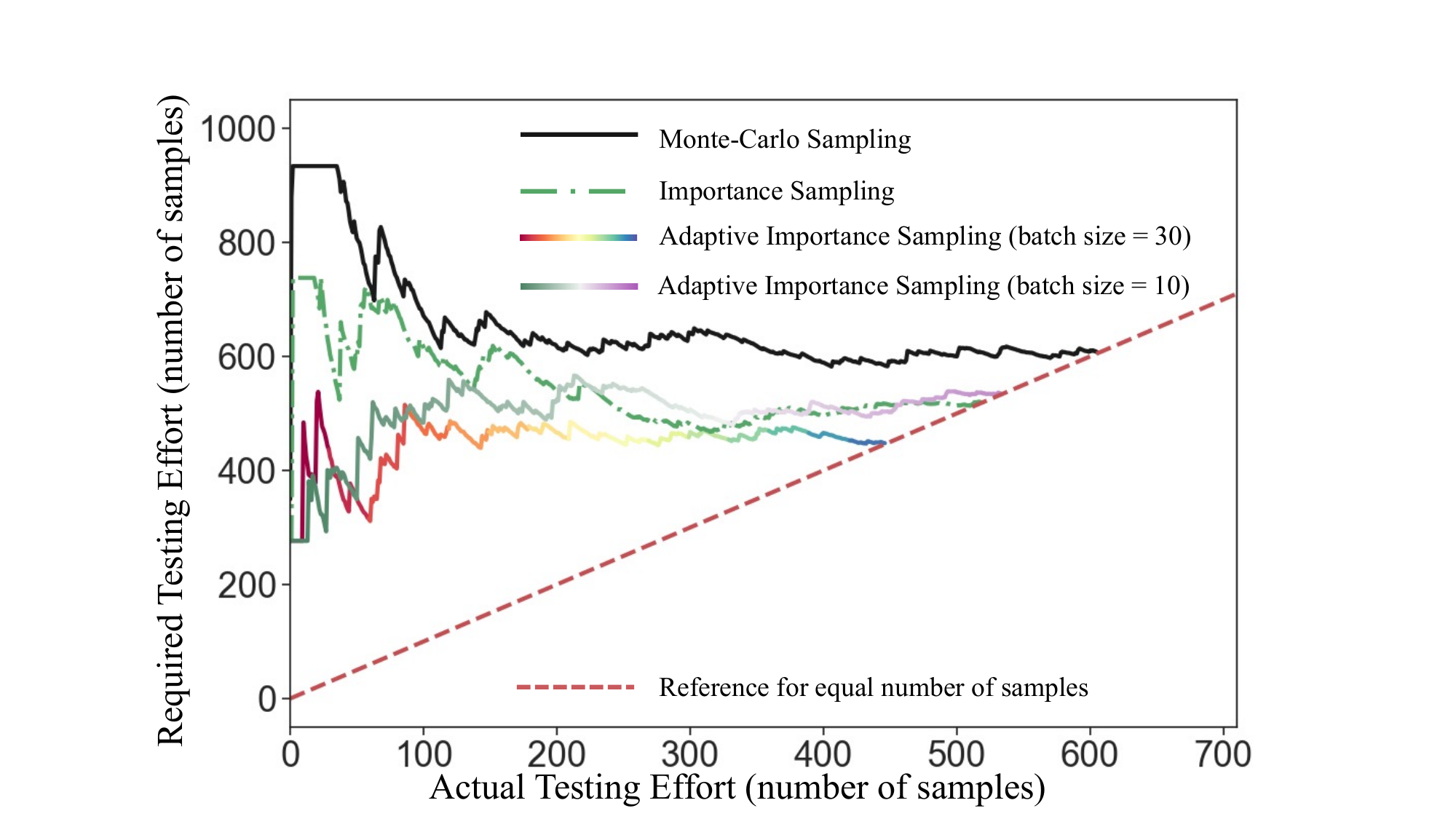}
      \caption{MuJoCo simulation-based tests.}
      \label{fig:g1_sim}
    \end{subfigure}%
    \hfill
    \begin{subfigure}{.49\textwidth}
      \centering
      \includegraphics[trim={3cm 0cm 4cm 1.5cm},clip,width=.99\textwidth]{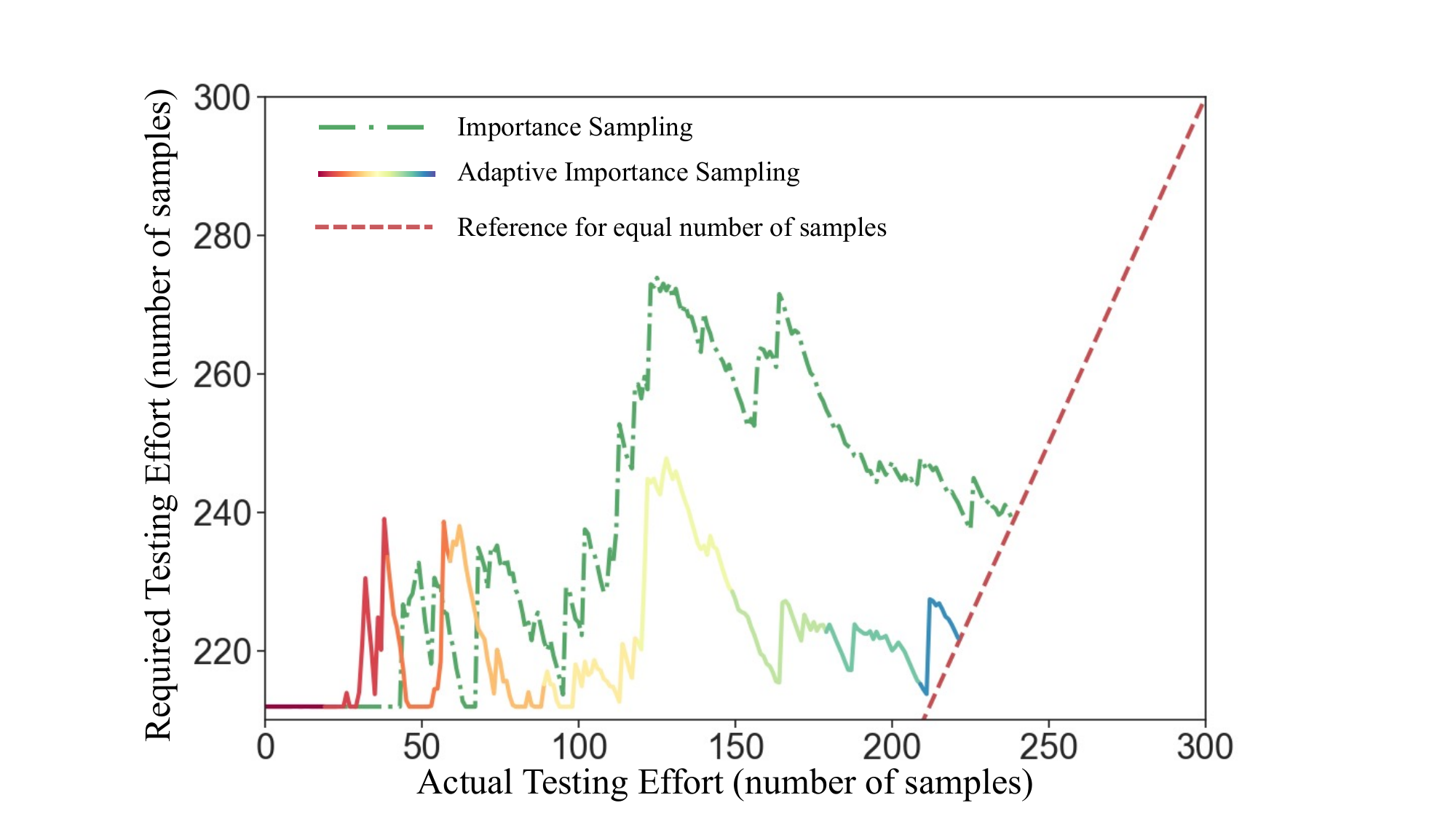}
      \caption{Real-world in-lab tests.}
      \label{fig:g1_real}
    \end{subfigure}%
    \hfill
    \caption{Comparing the testing efforts for the G1 locomotive performance estimate among trials of different $\alpha$-quantized Algorithm~\ref{alg:is_testing} variants in MuJoCo simulator (the original subject robot policy was trained in Issac Gym, but is tested in MuJoCo) and in real-world. For the AIS algorithms shown in both subfigures, batch transitions are visually indicated using gradually shifting color gradients, each following a distinct color palette.}
    \label{fig:g1_loco}
    \vspace{-1mm}
\end{figure*}

We start with a baseline implementation following the Monte Carlo sampling strategy by $p$. We consider two additional variants. The first is an importance sampling variant using a fixed proposal distribution $q$ parameterized as a Beta distribution~\cite{gupta2004handbook}, whose probability density function (PDF) over the bounded interval $[0,1]$ is
\begin{equation}
    f_{B}(x; a, b) = \frac{1}{B(a,b)}x^{a-1}(1-x)^{b-1}
\end{equation}
and $B(a,b)$ is the Beta function defined as
\begin{equation}
    B(a,b) = \int_0^1 t^{a-1} (1-t)^{b-1} dt
\end{equation}
The PDF can be extended to any bounded interval via a linear transformation. For the aforementioned command set, we employ three such Beta distributions—one for each command dimension. The distribution parameters were fitted using Monte Carlo samples collected from 300 simulation-based tests. Due to the high cost and practical difficulty of performing sufficient Monte Carlo sampling in real-world settings, these fitted parameters were reused in the real-world tests without re-fitting. The resulting shape parameters are $\textbf{a} = [0.83, 0.73, 0.77]$ and $\textbf{b} = [0.79, 0.75, 0.76]$. Each distribution is U-shaped, with higher density near the boundaries and lower density near the center, supporting our earlier conjecture that the distribution is not uniform over the command space.

The second variant is an AIS (adaptive importance sampling) algorithm. The proposal distribution is a mixture: with probability $0.1$, samples are drawn from the original distribution $p$; with probability $0.9$, they are drawn from a Beta-distribution-based proposal $q$. Such a mixed-sampling strategy is commonly employed in importance sampling methods to balance exploration and exploitation, thereby avoiding premature convergence or excessive bias toward high-probability regions~\cite{feng2020testing,zhao2016accelerated}. The parameters of $q$ are updated every $d$ steps by fitting a Beta distribution to the most recent batch of $d$ collected samples. Specifically, let $(a, b)$ denote the current shape parameters of the Beta distribution, and let $(a', b')$ be the parameters fitted to the latest batch of $d$ samples. The update rule follows an exponential moving average:
\begin{equation}
    a \leftarrow (1 - l_r)\,a + l_r\,a', \quad
b \leftarrow (1 - l_r)\,b + l_r\,b',
\end{equation}
where $l_r = 0.1$ is the learning rate. The initial proposal $q$ is parameterized with shape values $a = b = 0.99$ for all entries (to emulate an almost uniform distribution in the bounded interval).

We execute the aforementioned algorithmic variants across multiple trials in both simulation and real-world settings. For the real-world tests, we employ the OptiTrack motion capture system (tracker configuration shown in Fig.~\ref{fig:g1_trackers}) equipped with four Prime$^{\text{x}}$22 high-speed cameras to measure the robot's actual states, with a rated positional accuracy of 0.15 mm. Five motion trackers are mounted on the robot’s head and torso, which together form a rigid body throughout all tests. The geometric center of the polyhedron defined by these five trackers is calibrated to align with a reference point on the robot's pelvis, ensuring correspondence with the commanded state of the G1 robot. The raw position readings are acquired at approximately 125 Hz, and velocity is computed through direct finite differencing between consecutive position measurements. No filtering or smoothing operations are applied beyond this basic numerical differentiation.

Across all tests, we use a consistent set of hyperparameters: $\gamma = 0.04$, $c = 0.05$, $\beta = 0.1$, and $m\bar{\omega} = 0.7$. These parameters are chosen not because they reflect the theoretical ideal, but because more stringent settings would require an impractically large number of tests in real-world conditions. Practical constraints—such as motion capture system setup and calibration, robot warm-up, resets between trials, and ensuring the battery remains above 20\%, among many others—limit the number of feasible test executions. These considerations emphasize the necessity of sample-efficient testing algorithms, particularly in real-world deployments.

\noindent\textbf{Results \& discussions}: Fig.~\ref{fig:g1_sim} presents the results of four different variants of the $\alpha$-quantized Algorithm\ref{alg:is_testing}, each employing a distinct mechanism for exploring the commanded states during testing. These include: (i) standard Monte Carlo sampling, (ii) importance sampling (using the aforementioned configurations and hyperparameters), (iii) AIS with batch size $d = 30$, and (iv) AIS with batch size $d = 10$. Each algorithm was executed for 10 independent trials. Despite differences in sampling behavior, all 40 trials converged to the same final tracking performance estimate of approximately $0.47$. However, as shown in Fig.~\ref{fig:g1_sim}, the number of required samples varied significantly: Monte Carlo sampling used up to 650 tests, while the most efficient AIS variant ($d = 30$) achieved the same estimate with only 442 tests.

In the real-world experiments, we employed two testing variants: (i) importance sampling and (ii) AIS with a batch size of 30. Every effort was made to ensure consistency in testing conditions across both trials. The two tests were conducted on the same subject robot, at the same location, spaced seven days apart. The importance sampling trial spanned three days and concluded after 240 tests. The AIS trial was completed in a single day and terminated after 222 tests. The motion capture system was re-calibrated prior to each trial, and additionally once during the AIS trial (no mid-test calibration was performed during the importance sampling trial). Despite these differences, both trials yielded the same final tracking performance estimate of~$0.39$. 

According to Theorem~\ref{thm:var_rp} and the specified parameters, both the simulation and real-world estimates are guaranteed to lie within $0.078$ of the ground-truth value with 95\% confidence and are repeatable with 90\% probability. Notably, the same locomotion algorithm exhibits a clear discrepancy in tracking performance between the simulator and the real world. The proposed testing framework offers a principled and parameterized way to detect and quantify such discrepancies, complete with formal guarantees. While explaining the root causes of this sim-to-real gap is beyond the scope of this paper, our results demonstrate the practical utility of the proposed method in reliably identifying such differences.

\section{Conclusion and Future Work}

In this paper, we presented a principled framework for achieving guaranteed repeatable, accurate, and efficient evaluation of robot performance within the class of statistical query-based testing algorithms. We demonstrated the method’s advantages and unique capabilities across a diverse set of applications, highlighting its robustness, generality, and practical relevance in both simulated and real-world testing scenarios.

We believe the proposed framework holds significant practical value beyond the experiments presented in Section~\ref{sec:exp}. Its parameterized structure and formal guarantees make it well-suited for rigorously comparing robotic systems under shared operational conditions or for validating a given system against predefined performance thresholds. Furthermore, the method’s flexibility in the choice of testing algorithm, combined with its natural compatibility with performance drift and time-varying behaviors, makes it particularly applicable to real-world testing in standardized environments involving multiple stakeholders. Beyond robotics, the framework also has promising potential in other domains—such as repeatable assessment of human behavior, biostatistical studies, and broader empirical sciences—where controlled, reliable, and interpretable evaluation processes are critical.

On the other hand, several challenges and directions for future exploration remain. For instance, the bounds established in Theorems~\ref{thm:quant-acc} and~\ref{thm:quant-rep} are derived under worst-case assumptions. While this approach is advantageous in that it avoids imposing strong distributional assumptions, it can also lead to conservative estimates. This limitation is evident empirically: as shown in the results of Section~\ref{sec:exp}, there is often only a marginal gap between the theoretical repeatability guarantee (based on $\beta$) and the empirically observed repeatability probability. This suggests room for refining the bounds through distribution-aware or data-adaptive techniques.

\appendices

\section{}\label{apx:notation}
\textbf{Notation: } The set of real and positive real numbers are denoted by $\R$ and $\R_{>0}$ respectively. $\Z$ denotes the set of positive integers. $\norm{\cdot}$ is the $\ell_2$-norm.

\nomenclature{$\s$}{Testing system state vector}
\nomenclature{$S$}{Testing system state space}
\nomenclature{$s$}{Testing system state space dimension}
\nomenclature{$\uu$}{Controllable actions of testing}
\nomenclature{$U$}{Testing action space}
\nomenclature{$u$}{Testing system action space dimension}
\nomenclature{$\omega \in W$}{Disturbance and uncertainty term}
\nomenclature{$f(\cdot)$}{Testing system dynamics function}
\nomenclature{$\s(0)$}{Initial testing state vector}
\nomenclature{$p_s$}{Probability distribution of states}
\nomenclature{$T$}{State trajectory collected during testing}
\nomenclature{$\xi$}{Length (time steps) of a test trajectory}
\nomenclature{$\bar{\uu}$}{Finite sequence of testing actions}
\nomenclature{$p_u$}{Probability distribution over actions (for open-loop testing)}
\nomenclature{$\pi$}{Feedback testing policy}
\nomenclature{$x$}{Composite vector $(\s_0, \bar{\uu})$, test cases}
\nomenclature{$p$}{Target probability distribution of test cases}
\nomenclature{$q$}{Proposal (importance) distribution of test cases}
\nomenclature{$\psi(\cdot)$}{Performance evaluation function}
\nomenclature{$\mathcal{M}$}{Bounded interval of possible performance measure values}
\nomenclature{$\overline{m}$}{Upper bound of possible performance measure values}
\nomenclature{$\underline{m}$}{Lower bound of possible performance measure values}
\nomenclature{$m$}{Range of the performance evaluation function}
\nomenclature{$r^*$}{Ground-truth performance measure}
\nomenclature{$r_n$}{Empirical performance estimate after $n$ tests}
\nomenclature{$x_i$}{Sample drawn from distribution $q$ at iteration $i$}
\nomenclature{$q_i$}{Adaptive proposal distribution at iteration $i$}
\nomenclature{$w_i(x)$}{Importance weight for sample $x$ at iteration $i$}
\nomenclature{$\overline{w}$}{Upper bound for the importance weight}
\nomenclature{$\mathcal{T}$}{Termination condition function for the testing algorithm}
\nomenclature{$\mathbb{B}$}{A Boolean set of $\{0,1\}$}
\nomenclature{$\mathcal{TE}$}{A statistical query-type performance testing algorithm}
\nomenclature{$\gamma$}{Accuracy tolerance}
\nomenclature{$c$}{Confidence level parameter for accuracy ($1-c$ for confidence level})
\nomenclature{$\beta$}{Confidence level parameter for accuracy repeatability}
\nomenclature{$\alpha$}{Quantization parameter for partitioning outcome space}
\nomenclature{$\hat{\sigma}_n$}{Empirical estimate of variance after $n$ samples}
\nomenclature{$v^d_x$}{The commanded longitudinal velocity from Section~\ref{sec:exp-loco}}
\nomenclature{$v^d_y$}{The commanded lateral velocity from Section~\ref{sec:exp-loco}}
\nomenclature{$\theta^d$}{The commanded yaw rate from Section~\ref{sec:exp-loco}}
\nomenclature{$\s^d$}{The commanded state from Section~\ref{sec:exp-loco}}
\nomenclature{$\s^o_i$}{The observed state at the $i$-th step from Section~\ref{sec:exp-loco}}
\nomenclature{$T^d$}{The commanded trajectory of states from Section~\ref{sec:exp-loco}}
\nomenclature{$T^o$}{The observed trajectory of states from Section~\ref{sec:exp-loco}}
\nomenclature{$a, b$}{Coefficients of the Beta-distribution}
\nomenclature{$f_B$}{The probability density function of Beta distribution}
\nomenclature{$B$}{The Beta function}
\nomenclature{$l_r$}{Learning rate for adaptive importance sampling}

\printnomenclature

\section{Proof of Theorem~\ref{thm:quant-rep}}\label{apx:thm2}
\begin{proof}
    To achieve $\beta$-repeatability, by \eqref{eq:beta-2}, we require
    \[
    (1-c)^2 \frac{4\gamma\alpha - \alpha^2}{4\gamma^2} \geq 1 - \beta.
    \]
    
    Multiplying both sides by $4\gamma^2$ and rearranging terms, we obtain
    \[
    (1-c)^2(4\gamma\alpha - \alpha^2) \geq 4\gamma^2(1 - \beta).
    \]
    
    The discriminant of this quadratic on L.H.S. must be non-negative, i.e.,
    \[
    \Delta = \bigl(-4\gamma(1-c)^2\bigr)^2 - 4(1-c)^2\cdot4\gamma^2(1-\beta) \geq 0.
    \]
    
    Simplifying the discriminant, we get $16\gamma^2(1-c)^4 - 16\gamma^2(1-c)^2(1-\beta)\geq0$. Dividing by $16\gamma^2(1-c)^2$ (noting these terms are positive), we obtain $\quad (1-c)^2\geq (1-\beta)$, matching precisely the stated requirement in Theorem~\ref{thm:quant-rep}.
    
    Thus, provided $(1-c)^2\geq (1-\beta)$, we can explicitly solve for $\alpha$ by applying the quadratic formula as
    \[
    \alpha = \frac{4\gamma(1-c)^2 \pm \sqrt{\Delta}}{2(1-c)^2}
    = 2\gamma\frac{(1-c)\pm\sqrt{(1-c)^2-(1-\beta)}}{(1-c)}.
    \]
    Choosing the smaller solution (since smaller $\alpha$ yields better accuracy and any $\alpha$ value in between the two solutions guarantees repeatability), we obtain exactly the form in \eqref{eq:alpha}, thus proving Theorem~\ref{thm:quant-rep}.
\end{proof}

\section{A Lyapunov perspective of the $\gamma$-accuracy of Algorithm~\ref{alg:is_testing}}\label{apx:lya}
\begin{theorem}(Adapted from \cite[Section~5.4.1, pp145--146]{kushner2003stochastic}: convergence w.p.1 as in Theorem~4.1, and drift form as in the “Comment on the proof”.)
\label{thm:lyapunov}
    For a  Discrete-Time Stochastic Dynamic System of the form
    \begin{equation}
        x_{n+1} = x_n + \alpha_n H(x_n,\xi_{n+1}), n=0,1,2,\ldots,
    \end{equation}
    $\{\xi_n\}$  is a sequence of random disturbances adapted to the filtration $\{\mathcal{F}_n\}$ with $x_n$ being $\mathcal{F}_n$-measurable.
    Let $\alpha_n$ be deterministic step-sizes satisfying the Robbins–Monro conditions:
    \begin{equation}\label{eq:robbins-monro}
        \alpha_n > 0, \sum_0^{\infty} \alpha_n = \infty, \sum_0^{\infty} \alpha_n^2 < \infty.
    \end{equation}
    If there exists a measurable function $f:\R^d \rightarrow \R^d$ with a fixed equilibrium $x^*$, such that
    \begin{equation}
        \mathbb{E}[H(x_n,\xi_{n+1})\mid \mathcal{F}_n, x_n] = f(x_n), f(x^*)=0
    \end{equation}
    with bounded second moment:
    \begin{equation}
        \sup_x \mathbb{E}[\norm{H(x,\xi_{n+1})}^2\mid x_n=x] < \infty.
    \end{equation}
    Assume there exists a (continuously differentiable) Lyapunov function $V:\mathbb{R}^d\!\to\![0,\infty)$ satisfying the following conditions:
    \begin{itemize}
        \item $V(x^*) = 0$ and $V(x) > 0$ for all $x \neq x^*$; moreover, $V(x)$ is radially unbounded, i.e., $V(x)\to\infty$ as $\|x\|\to\infty$.
        \item There exist constants $w>0$ and $C<\infty$ such that, for all $x$ and $n$,
        \begin{equation}\label{eq:lya-drift}
            \begin{aligned}
                \mathbb{E}[V(x_{n+1}) & - V(x_n) \mid x_n = x] \\ &
            \leq -\,\alpha_n\, w\, \phi(\|x - x^*\|) + C \alpha_n^2,
            \end{aligned}
        \end{equation}
        where $\phi:[0,\infty)\!\to\![0,\infty)$ is nondecreasing with $\phi(0)=0$ and $\phi(r)>0$ for all $r>0$.
    \end{itemize}
    Then, $x_n \to x^*$ almost surely, which certifies the stability of $x^*$.
\end{theorem}
Let $x_n$ and $\alpha_n$ in Theorem~\ref{thm:lyapunov} be $r_n$ and $\frac{1}{n}$ in \eqref{eq:is-dtsd}, respectively. Consider the equilibrium $x^*$ as $r_* = \mathbb{E}_p[\psi(x)]$. We have the random noise function as $H(x_n, \xi_{n+1}) = \psi(x_{n+1})\frac{p(x_{n+1})}{q(x_{n+1})}-x_n$, leading to the standard form of $x_{n+1} = x_n + \alpha_nH(x_n, \xi_{n+1})$. Given $\frac{1}{n}$ clearly satisfies the Robbins–Monro Conditions, and the mean dynamics $f(x_n):= \mathbb{E}[H(x_n, \xi_{n+1})\mid \mathcal{F}_n, x_n] = \mathbb{E}_{q}\Big[ \psi(x_{n+1})\frac{p(x_{n+1})}{q(x_{n+1})}\Big] -x_n= r^*-x_n$, thus $f(r^*) = r^*-r^*=0$. Further more, the bounded second moment condition is satisfied as $\psi(\cdot)$ and $\frac{p(\cdot)}{q(\cdot)}$ are both bounded per statements in Section~\ref{sec:prob}. Consider the Lyapunov function $V(x) = (x-r^*)^2$. It is clear that $V(r^*)=0$, $V(x)>0, \forall x \neq r^*$, and $V(x)$ is radially unbounded since as $|x| \rightarrow \infty$, $(x-r^*)^2 \rightarrow \infty$. Finally, the conditional drift
\begin{equation}
    \begin{aligned}
        \mathbb{E}& [V(x_{n+1})-V(x_n)|x_n=x] \\ & = \mathbb{E}[(x+\alpha_nH(x, \xi_{n+1}))^2-(x-r^*)^2|x_n=x] \\ & = -2\alpha_n(x-r^*)^2 + \alpha_n^2\mathbb{E}[H(x, \xi_{n+1})^2\mid x_n=x] \\ & \leq -2\alpha_n(x-r^*)^2 + C\alpha_n^2
    \end{aligned}
\end{equation}
for some finite constant $C$. Given that $\alpha_n^2=1/n^2$ is summable, we satisfy the Lyapunov-type drift condition required by the theorem. This yields that the estimator $r_n$ converges to $r^*$ almost surely.
\bibliographystyle{IEEEtran}
\bibliography{output}

\end{document}